\documentclass[twoside]{article}

\usepackage[preprint]{aistats2023}
\usepackage{commands}
\newtheorem{definition}{Definition}
\newtheorem{theorem}{Theorem}

\begin{document}

%

%
\runningauthor{Stirn, Wessels, Schertzer, Pereira, Sanjana, Knowles}

\twocolumn[

\aistatstitle{Faithful Heteroscedastic Regression with Neural Networks}

\aistatsauthor{Andrew Stirn \And Hans-Hermann Wessels \And Megan Schertzer}
\aistatsaddress{Columbia University (CU) \And  New York Genome Center (NYGC) \And CU \& NYGC}

\aistatsauthor{Laura Pereira \And Neville E. Sanjana \And David A. Knowles}
\aistatsaddress{NYGC \And  New York University \& NYGC  \And CU \& NYGC}

]

\begin{abstract}

Heteroscedastic regression models a Gaussian variable's mean and variance as a function of covariates.
Parametric methods that employ neural networks for these parameter maps can capture complex relationships in the data.
Yet, optimizing network parameters via log likelihood gradients can yield suboptimal mean and uncalibrated variance estimates.
Current solutions side-step this optimization problem with surrogate objectives or Bayesian treatments.
Instead, we make two simple modifications to optimization.
Notably, their combination produces a heteroscedastic model with mean estimates that are provably as accurate as those from its homoscedastic counterpart (i.e.~fitting the mean under squared error loss).
For a wide variety of network and task complexities, we find that mean estimates from existing heteroscedastic solutions can be significantly less accurate than those from an equivalently expressive mean-only model.
Our approach provably retains the accuracy of an equally flexible mean-only model while also offering best-in-class variance calibration.
Lastly, we show how to leverage our method to recover the underlying heteroscedastic noise variance.

\end{abstract}


\section{INTRODUCTION}\label{sec:introduction}

Model uncertainty can be categorized as epistemic or aleatoric~\citep{der2009aleatory}.
Epistemic uncertainty is that which gathering more data or refining models can reduce.
For example, uncertainty in model parameters due to lacking data is epistemic.
Aleatoric uncertainty is that which more data or model refinement cannot reduce.
Measurement noise is a source of aleatoric uncertainty--growing a noisy dataset will eventually saturate model performance.
Ideally, a model's predictive variance accurately captures both epistemic and aleatoric uncertainty;
doing so enables active learning~\citep{gal2017deep}, reinforcement learning~\citep{osband2016deep,chua2018deep,yu2020mopo}, and identification of uncertain predictions.
For example, a biologist that uses a model to design high efficacy CRISPR guides might only consider those with low predictive uncertainty.

In regression, noise variance comes in two flavors.
Homoscedastic noise variance is constant across all data and can be represented with a fixed global parameter.
A model that only learns a mapping from covariates onto the response variable's mean via minimizing the sum of squared errors is implicitly homoscedastic.
Heteroscedastic noise variance changes w.r.t.~to covariates.
Regression models can capture heteroscedasticity in the data by learning a function from covariates to variance in addition to the function from covariates to the mean.

Mapping covariates onto the parameter space of a heteroscedastic Gaussian with neural networks and minimizing the negative log likelihood (NLL)~\citep{nix1994estimating} is the de-facto method underlying many modern approaches to predictive uncertainty estimation in regression.
\Citet{kendall2017uncertainties} use Monte Carlo dropout~\citep{gal2016dropout} to accumulate parameter estimates across sampled dropouts such that their
empirical distribution better captures predictive uncertainty.
\Citet{lakshminarayanan2017simple} do the same, but accumulate parameter estimates from an ensemble of models instead.
Despite its prevalence, minimizing the NLL of a heteroscedastic Gaussian has many pitfalls.
\Citet{takahashi2018student} identify optimization instabilities that occur when variance is driven towards zero for data whose mean estimates have near-zero error.
\Citet{skafte2019reliable} find that Monte Carlo and ensemble approaches can underestimate the true variance.
\Citet{seitzer2022on} show heteroscedastic variance estimation can degrade predictive mean performance.

\paragraph{Summary of Contributions.}

Consider a heteroscedastic neural network parameterizing a response variable's mean and covariance.
Eliminating all computations in the computational graph that are not ancestors of the mean output yields a mean-only model (a subnetwork of the original heteroscedastic model).
We propose using this mean-only model as the baseline for assessing the heteroscedastic model's mean estimates since both models have identical expressive power for mean estimation.
We label any heteroscedastic model optimized separately from its mean-only baseline as `faithful' if both models have equally accurate mean estimates.
`Faithfulness' requires that adding a covariance output head to a mean-only network and retraining it should not worsen its mean estimates.
We formalize this notion of `faithfulness' in \cref{def:faithful}, \cref{sec:methods}.

We propose two modifications to a heteroscedastic Gaussian model's NLL minimization that both save computations and are incredibly easy to implement.
Notably, their combination guarantees a `faithful' heteroscedastic model with mean estimates that are equivalent to its mean-only baseline's.
We experimentally confirm this claim and show our method achieves best-in-class variance calibration.

Current heteroscedastic methods combine aleatoric and epistemic uncertainty~\citep{kendall2017uncertainties,lakshminarayanan2017simple} in the predictive variance.
Our method offers reliable mean and variance estimates regardless of the true noise variance, which, for certain datasets, we can leverage to accurately decompose predictive variance into its aleatoric and epistemic components.
Accurately isolating aleatoric uncertainty enables scientists to study heteroscedasticity in observed systems.


\section{METHODS}\label{sec:methods}

\paragraph{Preliminaries.}
Homoscedastic regression models response variable $Y|X=x \sim \N(\mu(x), \Sigma)$.
Mean function $\mu(\cdot)$ operates on covariates $x$, while the (co)variance parameter has no dependence on $x$.
Heteroscedastic regressions models $Y|X=x \sim \N(\mu(x), \Sigma(x))$,
which establishes (co)variance as a function of covariates $\Sigma(x)$.
Homoscedasticity is often a convenient rather than accurate assumption.
Heteroscedastic models can capture homoscedasticity if $\Sigma(\cdot)$ learns a constant output that matches the true global (co)variance.
For brevity, we will use variance and covariance interchangeably henceforth.

\Citet{nix1994estimating} propose neural networks as the functional class for the mean and variance function(s);
this provides remarkable flexibility to capture complex relationships in the data.
To fit the mean and variance network(s) under their proposal, one minimizes the NLL, $\LL \equiv$
\begin{sizeddisplay}{\small}
\begin{align*}
    \frac{1}{2} \sum_{(x,y)\in\D} \log|2\pi\Sigma(x)| + \big(y - \mu(x)\big)^T\big(\Sigma(x)\big)^{-1}\big(y - \mu(x)\big).
\end{align*}
\end{sizeddisplay}
When $\mu(\cdot)$ and $\Sigma(\cdot)$ are from separate network classes
\begin{align}
    \mu(\cdot) &\in \F_\mu: (\X,\Theta_\mu) \rightarrow \Re^{\dim(\Y)} \label{eq:mean-func}\\
    \Sigma(\cdot) &\in \F_\Sigma: (\X,\Theta_\Sigma) \rightarrow \SM_{++}^{\dim(\Y) \times \dim(\Y)} \nonumber,
\end{align}
one can separately optimize network parameters $\theta_\mu$ and $\theta_\Sigma$ by computing gradients $\nabla_{\theta_\mu} \LL$ and $\nabla_{\theta_\Sigma} \LL$ and using stochastic gradient descent or, like we do, an adaptive gradient algorithm such as Adam~\citep{kingma2014adam}.
A more general treatment assumes a single network class
\begin{sizeddisplay}{\small}
\begin{align}
    (\mu,\Sigma)(\cdot) &\in \G: (\X,\Theta) \rightarrow (\Re^{\dim(\Y)}, \SM_{++}^{\dim(\Y) \times \dim(\Y)}) \label{eq:comb-func}.
\end{align}
\end{sizeddisplay}
Whenever $\Theta_\mu \cap \Theta_\Sigma = \emptyset$, $\mu(\cdot)$ and $\Sigma(\cdot)$ are separate networks.
Conversely, $\Theta_\mu \cap \Theta_\Sigma \neq \emptyset$, implies $\mu(\cdot)$ and $\Sigma(\cdot)$ share ancestral computations in $\G$'s computational graph that involve shared parameter space $\Theta_\mu \cap \Theta_\Sigma$.

\paragraph{The Problem.} Without loss of generality, consider a univariate response ($\dim(Y)=1$) where $\sigma^2(\cdot)$ is a neural network parameterizing a scalar variance.
Then, the gradients of the NLL w.r.t.~the parameterizing function outputs are
\begin{align}
    \nabla_{\mu(x)} \LL &= \sum_{(x,y)\in\D} \frac{\mu(x) - y}{\sigma^2(x)} \label{eq:grad-mean}\\
    \nabla_{\sigma^2(x)} \LL &= \sum_{(x,y)\in\D} \frac{\sigma^2(x) - (y - \mu(x))^2}{2(\sigma^2(x))^2} \label{eq:grad-variance}.
\end{align}
\Citet{skafte2019reliable} recognize the $\frac{1}{\sigma^2(x)}$ scaling in both gradients quickens learning for low-variance points and thus biases performance towards regions of data with low noise.
\Citet{seitzer2022on} further recognize that learning is biased not just against data with higher true variance but also against points whose mean predictions are poor.
Here, a model may use high variance (regardless of the true variance) to explain poor mean estimates instead of improving them;
this creates a `rich-get-richer' dynamic, where points with lower predictive variance continuously provide the largest learning signal.
Please see \citet{seitzer2022on} for a wonderful exposition of this phenomena.

\paragraph{Assessing Mean Estimation.}
To fairly assess mean estimation accuracy of a heteroscedastic model from $\G$, we propose constructing a mean-only baseline from $\F_\mu$ satisfying \cref{def:mean-only}.
\begin{definition}\label{def:mean-only}
    For a heteroscedastic model from network class $\G$ with parameter space $\Theta$ (\cref{eq:comb-func}), let $\F_\mu$ with parameter space $\Theta_\mu$ be the computational subgraph of $\G$ that removes all computational nodes that are not ancestors of the mean output.
    Clearly then, $\F_\mu \subset \G$ and $\Theta_\mu \subset \Theta$.
\end{definition}
By construction, $\G$ and $\F_\mu$ are equivalently powerful at mean estimation.
The decision to use a mean-only model from $\F_\mu$ as a baseline is well motivated for a couple of reasons.
First, prevailing evidence suggests the additional fitting of variance is to blame for poor predictive performance when simultaneously fitting the mean and variance functions of a heteroscedastic model via gradient-based NLL minimization~\citep{takahashi2018student,skafte2019reliable,seitzer2022on}.
Second, mean-only models that minimize the sum of squared errors (SSE) are equivalently homoscedastic with isotropic unit covariance ($\Sigma=I$) and thus cannot get stuck in the cycle of using high local variance to explain poor mean estimates since they equally scale gradients by the inverted global covariance $\Sigma^{-1}=I$.

Mean square error (MSE) or root MSE (RMSE) are common assessments of mean estimation accuracy.
We propose a desideratum for a heteroscedastic model is that its MSE is no worse than its mean-only baseline's MSE.
If a heteroscedastic model meets this criteria, we say it is `faithful.'
\Cref{def:faithful} formalizes this notion of faithfulness.

\begin{definition}\label{def:faithful}
    Let $(\mu,\Sigma)(\cdot) \in \G$ (per \cref{eq:comb-func}) with initial parameters $\theta\in\Theta$ map covariates onto mean estimate $\mu(x;\theta) \approx \E[Y|X=x]$ and covariance estimate $\Sigma(x;\theta) \approx \text{Cov}[Y|X=x]$.
    Let $\mu_0(\cdot) \in \F_\mu$ (satisfying \cref{def:mean-only}) with initial parameters $\theta_{\mu} \subset \theta$ map covariates onto mean estimate $\mu_0(x;\theta_{\mu}) \approx \E[Y|X=x]$.
    If $\mu_0(\cdot)$ and $(\mu,\Sigma)(\cdot)$'s parameters are separately optimized with algorithm $\A$ and
    $\E[|Y-\mu_0(x)|_2 - |Y-\mu(x)|_2] \geq 0$,
    then $\G$ is \textbf{faithful} to $\F_\mu$ under $\A$.
    Otherwise, $\G$ is \textbf{unfaithful} to $\F_\mu$ under $\A$.
    The expectation is w.r.t.~the data distribution $p(X,Y)$ and all sources of randomness in $\A$.
\end{definition}

Using separate mean and variance networks, minimizing the SSE of the mean network first, freezing its parameters, and thereafter optimizing the NLL of the variance network achieves faithfulness under \cref{def:faithful}.
This approach has several key deficiencies.
It requires monitoring the mean network for convergence in order to transition to fitting the variance network.
Sequential optimization is both inefficient and can lead to poor results if mean network optimization halts prematurely.
Also, mandating separate networks for the mean and variance prohibits sharing of learned representations.
For example, if data points are natural images and the mean network uses convolutional layers, it is likely that the first convolution layer will learn a set of edge detectors that might be redundant to those learned by a convolutional layer in a separate variance network.
Relearning representations is computationally wasteful and only gets worse with deeper networks, which are often those with the best performance.
Our proposal provably satisfies \cref{def:faithful} while also addressing these deficiencies.

\paragraph{Our Solution.}
Consider the highly general network for parameterizing a Gaussian heteroscedastic model in \cref{fig:comp-graph}, which we partition into three subnetworks:
\begin{itemize}
    \item Shared representation learner $f_{\text{trunk}}$ has parameters $\theta_z$.
    \item Mean parameter head $f_\mu$ has parameters $\theta_\mu$.
    \item Covariance parameter head $f_\Sigma$ has parameters $\theta_\Sigma$.
\end{itemize}
These three subnetworks comprise a hetersocedastic network from $\G$ (\cref{eq:comb-func}, outer gray box \cref{fig:comp-graph}) such that $(\theta_z,\theta_\mu,\theta_\Sigma)\in\Theta$.
Subnetworks $f_{\text{trunk}}$ and $f_\mu$ comprise a mean-only network from $\F_\mu$ (\cref{eq:mean-func}, inner gray ellipse \cref{fig:comp-graph}) such that $(\theta_z,\theta_\mu)\in\Theta_\mu$.
Thus, $\F_\mu$ in \cref{fig:comp-graph} is the mean-only baseline satisfying \cref{def:mean-only} to which we will compare a heteroscedastic model from $\G$.
If $f_{\text{trunk}}$ is the identity function, then the mean and covariance networks are separate.
Alternatively, $f_{\text{trunk}}$ could be a multi-layer network with vector output $z$.
One possible $f_\Sigma(z;w_\Sigma,b_\Sigma)=\text{diag}(\text{softplus}(z^T w_\Sigma + b_\Sigma))$ adds a single output layer.

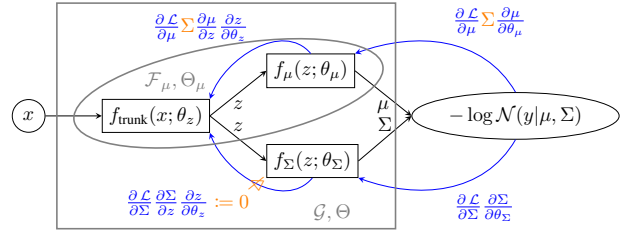
\begin{figure}[h]
\resizebox{0.475\textwidth}{!}{%
\begin{tikzpicture}
    \node[circle,draw=black] (x) {$x$};
    \node[rectangle,draw=black] (trunk) [right=of x] {$f_{\text{trunk}}(x;\theta_z)$};
    \node[rectangle,draw=black] (mean) [right=of trunk, yshift=0.8cm] {$f_\mu(z;\theta_\mu)$};
    \node[rectangle,draw=black] (variance) [right=of trunk, yshift=-0.8cm] {$f_\Sigma(z;\theta_\Sigma)$};
    \node[ellipse,draw=black] (loss) [right=of mean, yshift=-0.8cm] {$-\log\N(y|\mu,\Sigma)$};

    \draw[-stealth] (x.east) -- (trunk.west);
    \draw[-stealth] (trunk.east) -- (mean.west) node[midway,below] {$z$};
    \draw[-stealth] (trunk.east) -- (variance.west) node[midway,above] {$z$};
    \draw[-stealth] (mean.east) -- (loss.west) node[midway,below] {$\mu$};
    \draw[-stealth] (variance.east) -- (loss.west) node[midway,above] {$\Sigma$};

    \draw[-stealth,draw=blue] (loss.north) to[out=120,in=20]
        node[auto,swap]
        {$\color{blue}\frac{\partial \LL}{\partial \mu}\color{orange}\Sigma\color{blue}\frac{\partial \mu}{\partial \theta_\mu}$}
        (mean.north east);
    \draw[-stealth,draw=blue] (loss.south) to[out=-120,in=-20]
        node[auto]
        {$\color{blue}\frac{\partial \LL}{\partial \Sigma}\frac{\partial \Sigma}{\partial \theta_\Sigma}$}
        (variance.south east);
    \draw[-stealth,draw=blue] (mean.north) to[out=130,in=70]
        node[auto,swap]
        {$\color{blue}\frac{\partial \LL}{\partial \mu}\color{orange}\Sigma\color{blue}\frac{\partial \mu}{\partial z}\frac{\partial z}{\partial \theta_z}$}
        (trunk.north east);
    \draw[-stealth,draw=blue] (variance.south) to[out=-130,in=-70]
        node[sloped] {\color{orange}$\ntriangleleft$}
        node[auto]
        {$\color{blue}\frac{\partial \LL}{\partial \Sigma}\frac{\partial \Sigma}{\partial z}\frac{\partial z}{\partial \theta_z}\color{orange}\coloneqq0$}
        (trunk.south east);

    \draw[gray,thick] let \p1=(trunk.west), \p2=(mean.east), \n1={atan2(\y2-\y1,\x2-\x1)}, \n2={veclen(\y2-\y1,\x2-\x1)}
        in ($ (\p1)!0.5!(\p2) $) ellipse [x radius=\n2/2+.5cm, y radius=.85cm, rotate=\n1];
    \node (F) [above=0pt of trunk, xshift=10pt] {$\color{gray}\F_\mu,\Theta_\mu$};

    \draw[gray,thick] (0.5,-2) rectangle (6.5,2);
    \node (G) [below=10pt of variance, xshift=10pt] {$\color{gray}\G,\Theta$};
\end{tikzpicture}
}%
\caption{Heteroscedastic Network Partitions}
\label{fig:comp-graph}
\end{figure}

Conventional optimization of \cref{fig:comp-graph} forward passes covariates $x$ (black arrows) to compute the NLL and then backward passes its gradient w.r.t.~network parameters (blue arrows and equations) to render parameter updates~\citep{nix1994estimating}.
We propose two changes (shown in orange) to gradient computation and its backward pass:
\begin{enumerate}[start=1,label={\bfseries Proposal \arabic*:},leftmargin=5.35em]
    \item Scale $\nabla_{\mu(x)} \LL$ by $\Sigma$, its inverse Jacobian (i.e.~a Newton step instead of a gradient step).
    \item Stop $\nabla_{\Sigma(x)} \LL$ from contributing to updates for any shared parameters $\theta_z$ in $f_{\text{trunk}}$.
\end{enumerate}
Together, these modifications ensure the mean subnetwork $\mu(x) = (f_\mu \circ f_\text{trunk})(x)$ will undergo optimization identically as if it were removed from $\G$ and trained on its own.
Hence, our solution guarantees $\G$ will be faithful to $\F_\mu$ according to \cref{def:faithful}.
Our first modification replaces the mean's gradient (\cref{eq:grad-mean}) with its second-order Newton step
\begin{align*}
    \sum_{(x,y)\in\D} \Sigma(x) \Sigma^{-1}(x) (\mu(x) - y) =\sum_{(x,y)\in\D} \mu(x) - y,
\end{align*}
which is recognizable as the gradient of a homoscedastic model with isotropic unit covariance (i.e.~from minimizing SSE).
The $\Sigma(x) \Sigma^{-1}(x)$ cancellation saves a matrix inversion and matrix-vector product.
Our second modification 
\begin{align*}
    \frac{\partial \LL}{\partial \theta_z} =
        \Bigg(\frac{\partial \LL}{\partial \mu} \frac{\partial \mu}{\partial z} +
            \cancelto{\coloneqq 0}{\frac{\partial \LL}{\partial \Sigma} \frac{\partial \Sigma}{\partial z}}\Bigg)
        \frac{\partial z}{\partial \theta_z}
\end{align*}
eliminates the covariance's influence on any shared parameters $\theta_z$.
Thus, by shielding $\theta_z$ from $\frac{\partial \LL}{\partial \Sigma} \frac{\partial \Sigma}{\partial z}\frac{\partial z}{\partial \theta_z}$ (the influence from $\Sigma(x)$) and ensuring equivalent parameter updates $\Delta_{\theta_\mu}$ and $\Delta_{\theta_z}$ between a heteroscedastic model and its homoscedastic mean-only baseline, we have provably satisfied \cref{def:faithful}.
By equally weighting all error terms, our proposal cannot enter the deleterious `rich-get-richer' cycle where poor mean estimates are explained by high variance and subsequently ignored.
Minimizing $\LL_\text{ours} \equiv$
\begin{sizeddisplay}{\small}
\begin{align}
    \sum_{(x,y)\in\D} \frac{\big|y-\mu(x)\big|_2}{2} -\log \N \big(y;\lfloor\mu(x)\rfloor,\Sigma(\lfloor f_{\text{trunk}}(x) \rfloor)\big)
    \label{eq:loss-ours}
\end{align}
\end{sizeddisplay}
is an equivalent implementation of our two proposals.
We use $\lfloor\cdot\rfloor$ to denote the stop gradient operation, which converts its operand to a constant from the optimizer's perspective.
As before, $\mu(x) = (f_\mu \circ f_\text{trunk})(x)$.
\Cref{eq:loss-ours} is easy to implement in popular deep learning frameworks and is compatible with any gradient-based optimizer.

\begin{theorem}\label{thm:faithful}
    Under the same assumptions of \cref{def:faithful}, gradient-based optimization of $\LL_\text{ours}$ (\cref{eq:loss-ours}) guarantees $\G$'s faithfulness to $\F_\mu$.
\end{theorem}

\begin{proof}
    Differentiating $\LL_\text{ours}$, one finds that
    \begin{align*}
        \frac{\partial \LL_\text{ours}}{\partial \theta_\mu} &= \sum_{(x,y)\in\D} \Big(f_\mu(z) - y\Big) \frac{\partial f_\mu(z)}{\partial \theta_\mu} \\
        \frac{\partial \LL_\text{ours}}{\partial \theta_z} &= \sum_{(x,y)\in\D} \Big(f_\mu(z) - y\Big) \frac{\partial f_\mu(z)}{\partial z} \frac{\partial z}{\partial \theta_z},
    \end{align*}
    which are equivalent to those arising from a SSE $\sum_{(x,y)\in\D} \frac{1}{2}\big|y-\mu(x)\big|_2$ loss.
    Thus, if both the heteroscedastic and homoscedastic optimizations experience the same randomness (e.g.~by using a random number seed), then parameter updates $\Delta\theta_\mu$ and $\Delta\theta_z$ will be identical at every optimization step.
    Because equivalence holds for every step for a single random number seed, it holds for all steps for any random number seed.
\end{proof}

\paragraph{Related Solutions.}

This article considers parametric methods that map a data point's covariates directly onto the response variable's parameter space.
In contrast, Gaussian processes (GPs) are non-parametric insofar that they require evaluating a kernel on all training (or inducing) points to compute the predictive mean and variance for a query point~\citep{rasmussen_gaussian_2006}.
Most GPs including those that use neural networks to project covariates to a latent space before kernel evaluation~\citep{wilson2016deep} assume homoscedastic isotropic noise $\sigma^2 I$ and cannot capture heteroscedastic noise variance (like we do).
While a homoscedastic GP's predictive variance does change w.r.t.~covariates, this change is in epistemic uncertainty only (e.g.~$x_*$'s distance to training/inducing points).
Technically, heteroscedasticity refers to noise variance only.
\Citet{goldberg1997regression} introduce a GP with heteroscedastic noise assumptions.
That said, GPs are a fundamentally different model class than the parametric methods we consider.
We only compare performances within model classes to avoid confounding model selection and optimization methods.

\citet{nix1994estimating} first proposed using neural networks as Gaussian parameter maps.
As such, we include their proposal in all of our experiments.
Their technique, despite its optimization pitfalls~\citep{seitzer2022on}, underlies Monte Carlo dropout~\citep{kendall2017uncertainties} and Deep Ensembles~\citep{lakshminarayanan2017simple}, which average outputs from \citet{nix1994estimating} over dropout samples and random initializations, respectively.
Both of these methods produce a uniform mixture of Normals as the predictive distribution.
In \cref{sec:experiments}, this underlying model experimentally underperforms, and we do not trust averaging to fix it.
Instead, our supplement adapts our proposals to these two additional model classes and finds significant performance gains.

Several approaches adopt a Bayesian perspective to alleviate optimization problems, which results in a Student predictive distribution.
\citet{takahashi2018student} parameterize a Student using mean, precision, and degrees-of-freedom network(s).
Integration over Gamma-distributed precisions seemingly improves optimization stability for points whose mean errors are small since integration almost surely includes non-zero variances.
\Citet{skafte2019reliable} use the same Student parameterization with several other modifications.
\Citet{stirn2020variational} place a prior over precision and perform variational inference.
For well-chosen priors, the resulting Kullback–Leibler (KL) divergence in the variatonal objective can regularize variance away from very small (those that cause optimization instability~\citep{takahashi2018student}) and very large (those that produce poor estimates via the `rich-get-richer' cycle~\citep{seitzer2022on}) variances.
Our supplement adapts our proposals to the Student model and again finds substantial performance gains.

Within the class of parametric methods with a Normal predictive, \citet{seitzer2022on} propose their $\beta$-NLL loss
\begin{sizeddisplay}{\small}
\begin{align*}
    \sum_{(x,y)\in\D} \lfloor \sigma^{2\beta}(x) \rfloor \Bigg(\frac{1}{2}\log \sigma^2(x) + \frac{(y - \mu(x))^2}{2\sigma^2(x)}\Bigg).
\end{align*}
\end{sizeddisplay}
Again, $\lfloor\cdot\rfloor$ is the stop gradient operation.
Setting $\beta=0$, is equivalent to conventional NLL minimization.
For $\beta=1$, the gradients of $\beta$-NLL and SSE w.r.t.~mean estimates are identical, but gradients w.r.t.~variance estimates will still influence shared parameters.
Changing $\beta$ from 0 to 1 changes the curvature of the gradient w.r.t.~variance function outputs (see supplement).
\citet{seitzer2022on} recommend $\beta=0.5$ for the best balance between RMSE and log likelihood performance, but never check variance calibration.
While $\beta=1$ seemingly prevents poor mean estimates explained by high variance from being ignored, it is important to test what effect, if any, $\beta\neq0$ has on variance calibration.
We examine mean estimation accuracy and variance calibration for $\beta\in\{0.5,1\}$ in our experiments.

\begin{figure*}[ht]
    \includegraphics[width=\textwidth]{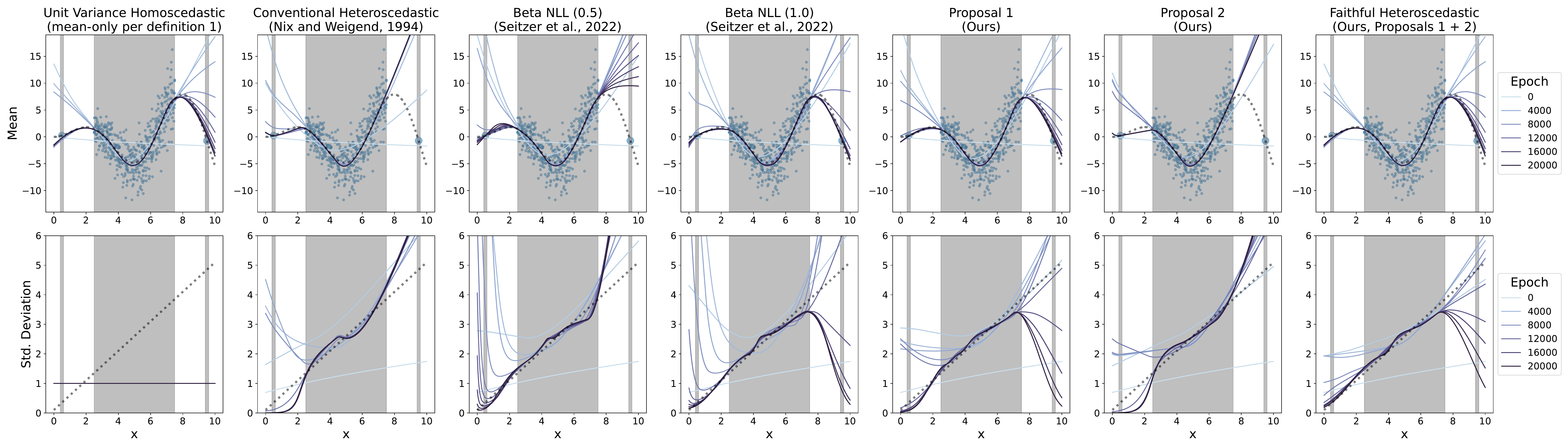}
    \caption{Convergence Behavior of Methods with Predictive Normal Distribtuions}
    \label{fig:covergence}
\end{figure*}

\section{EXPERIMENTS}\label{sec:experiments}
Our experiments examine mean accuracies, variance calibrations, and log likelihoods from our method and our chosen baselines.
\textbf{Unit Variance Homoscedastic} is the mean-only baseline (satisfying \cref{def:mean-only}) and uses a nominal homoscedastic isotropic unit covariance for variance calibration and log likelihood measurements.
\textbf{Conventional Heteroscedastic} denotes standard NLL minimization of a heteroscedastic model~\citep{nix1994estimating}.
We label results for \textbf{Beta NLL}~\citep{seitzer2022on} with the $\beta$ setting in parentheses.
\textbf{Proposal 1} and \textbf{Proposal 2} apply our two proposals from \cref{sec:methods} in isolation to confirm both are necessary.
Our top method, \textbf{Faithful Heteroscedastic}, utilizes proposals 1 and 2 for guaranteed faithfulness.

\subsection{Convergence of the Predictive Mean and Variance}\label{subsec:convergence}
\Cref{fig:covergence} plots the convergence behaviour for all methods with a Normal predictive distribution.
Our supplement has plots for Deep Ensemble, Monte Carlo Dropout, and Student based methods with similar findings.
We generate data similarly to~\citet{skafte2019reliable,stirn2020variational,seitzer2022on}.
We sample $x \sim \text{Uniform}(2.5,7.5)$ and then set $y \triangleq x \cdot \sin(x) + \epsilon$ with $\epsilon|x \sim \N(0, 0.1 + |0.5 x|)$.
We add $(X=0.5,Y=0.5 \cdot \sin(0.5))$ and $(X=9.5,Y=9.5 \cdot \sin(9.5))$ as isolated points for a total of 500 samples (shown in blue, isolated points appear larger for visual convenience).
We highlight regions of covariate space containing data with gray backgrounds.
We plot the predictive mean and variance every 2000 training epochs in the top and bottom rows of \cref{fig:covergence}, respectively.
The dashed lines show the true mean and variance.

This experiment employs a neural network with a single hidden layer for $f_{\text{trunk}}$.
The mean and variance heads are each a single layer.
The mean-only model, Unit Variance Homoscedastic, converges on both isolated points while estimating the true mean over the data rich interval $[2.5,7.5]$.
This behavior not only mimics a GP (to which this network architecture can have theoretical connections~\citep{rasmussen_gaussian_2006}) but also confirms the mean subnetwork of the heteroscedastic model has the flexibility to do the same.
However, the Conventional Heteroscedastic model, after quickly fitting its mean to the majority of points, has a drastic error at $X=9.5$.
As \citet{seitzer2022on} suspect, the model quickly increases variance to explain this large error, which effectively eliminates this point's learning signal and ultimately prohibits convergence thereto.
Our Proposal 2 does the same.
This problem does not fully disappear for Beta NLL (0.5).
The Beta NLL (1.0), Proposal 1, and Faithful Heteroscedastic models all produce accurate mean estimates with convergence at both isolated points and accurate variance estimates.

\subsection{UCI Regression}\label{subsec:uci}
The following experiments utilize publicly available regression datasets from the UCI repository\footnote{Our code downloads and processes these datasets automatically.}.
We scale response variables to have zero mean and unit variance.
We now use a neural network with two hidden layers as the shared representation learner $f_\text{trunk}$.
The mean and variance heads are single layers operating on $f_\text{trunk}$'s output.
To best approximate the expectation in \cref{def:faithful}, we first divide each UCI dataset into ten folds.
We then aggregate the predictions of each held-out fold from the ten models resulting from cross-validation such that we have a held-out prediction for every data point~\citep{hofling2008study}.
Every model uses the same fold assignments and, for every fold, shares the same initial parameters.
Our supplement contains additional details.

\begin{table*}[t!]
    \caption{UCI Regression Performance of Methods with Predictive Normal Distribtuions}
    \label{tab:uci}
    \adjustbox{width=\textwidth}{\begin{tabular}{l|ccc|ccc|ccc|ccc|ccc|ccc|ccc}
\toprule
 & \multicolumn{3}{|c}{Unit Variance Homoscedastic} & \multicolumn{3}{|c}{Conventional Heteroscedastic} & \multicolumn{3}{|c}{Beta NLL (0.5)} & \multicolumn{3}{|c}{Beta NLL (1.0)} & \multicolumn{3}{|c}{Proposal 1} & \multicolumn{3}{|c}{Proposal 2} & \multicolumn{3}{|c}{Faithful Heteroscedastic} \\
 & \multicolumn{3}{|c}{(mean-only per \cref{def:mean-only})} & \multicolumn{3}{|c}{\citep{nix1994estimating}} & \multicolumn{3}{|c}{\citep{seitzer2022on}} & \multicolumn{3}{|c}{\citep{seitzer2022on}} & \multicolumn{3}{|c}{(Ours)} & \multicolumn{3}{|c}{(Ours)} & \multicolumn{3}{|c}{(Ours, Proposals 1 + 2)} \\
Dataset & RMSE & ECE & LL & RMSE & ECE & LL & RMSE & ECE & LL & RMSE & ECE & LL & RMSE & ECE & LL & RMSE & ECE & LL & RMSE & ECE & LL \\
\midrule
boston & 0.304 & 0.168 & -0.965 & \sout{0.374} & \sout{0.00977} & \sout{-2.04} & \textbf{0.341} & \textbf{0.0224} & \textbf{-12} & \textbf{0.335} & \textbf{0.0269} & \textbf{-4.5} & \sout{0.355} & \sout{0.00973} & \sout{-0.886} & \sout{0.355} & \sout{0.0121} & \sout{-2.27} & \textbf{0.304} & 0.0303 & \textbf{-17.4} \\
carbon & 0.0489 & 0.401 & -2.76 & \sout{0.0493} & \sout{0.00247} & \sout{3.89} & \textbf{0.0488} & 0.00152 & \textbf{-30.8} & \sout{0.05} & \sout{0.00727} & \sout{-9.76e+04} & \sout{0.0818} & \sout{5.48e-05} & \sout{4.78} & \sout{0.0491} & \sout{0.00225} & \sout{5.38} & \textbf{0.0489} & \textbf{0.00124} & \textbf{-9.49} \\
concrete & 0.265 & 0.127 & -0.954 & \sout{0.315} & \sout{0.00601} & \sout{-2.24} & \textbf{0.269} & 0.0348 & -59.8 & \textbf{0.263} & 0.0417 & -22.2 & \sout{0.293} & \sout{0.0103} & \sout{-2} & \sout{0.294} & \sout{0.00941} & \sout{-1.2} & \textbf{0.265} & \textbf{0.0282} & \textbf{-2.81} \\
energy & 0.16 & 0.314 & -1.85 & \sout{0.193} & \sout{0.00236} & \sout{3.39} & \sout{0.182} & \sout{0.00139} & \sout{3.9} & \textbf{0.168} & 0.00242 & \textbf{3.29} & \sout{0.195} & \sout{0.00161} & \sout{2.79} & \sout{0.193} & \sout{0.00966} & \sout{2.6} & \textbf{0.16} & \textbf{0.00127} & \textbf{3.35} \\
naval & 0.0247 & 0.401 & -1.84 & \sout{0.533} & \sout{0.00432} & \sout{3.23} & \sout{0.159} & \sout{0.0133} & \sout{1.85} & \sout{0.0262} & \sout{0.00116} & \sout{6.6} & \sout{0.207} & \sout{0.000175} & \sout{2.64} & \sout{0.581} & \sout{0.00212} & \sout{2.8} & \textbf{0.0247} & \textbf{0.00197} & \textbf{6.63} \\
power plant & 0.22 & 0.0105 & -0.943 & \sout{0.227} & \sout{0.000175} & \sout{0.0106} & \sout{0.224} & \sout{0.000199} & \sout{0.0456} & \sout{0.223} & \sout{0.00034} & \sout{-53.2} & \sout{0.235} & \sout{0.000155} & \sout{0.0385} & \sout{0.223} & \sout{0.000127} & \sout{0.0722} & \textbf{0.22} & \textbf{0.000183} & \textbf{0.0937} \\
protein & 0.00257 & 0.489 & -0.919 & \sout{0.029} & \sout{0.00291} & \sout{3.73} & \sout{0.00397} & \sout{0.00599} & \sout{4.57} & \sout{0.00308} & \sout{0.00342} & \sout{4.62} & \sout{0.0372} & \sout{8.67e-05} & \sout{2.1} & \sout{0.0675} & \sout{0.00957} & \sout{2.67} & \textbf{0.00257} & \textbf{0.00591} & \textbf{4.68} \\
superconductivity & 0.326 & 0.0104 & -0.972 & \sout{0.388} & \sout{0.00109} & \sout{-10} & \textbf{0.327} & 0.002 & -23.7 & \textbf{0.327} & 0.00184 & -7.57 & \sout{0.37} & \sout{0.000297} & \sout{-0.181} & \sout{0.358} & \sout{0.00174} & \sout{-1.51} & \textbf{0.326} & \textbf{0.00103} & \textbf{-0.291} \\
wine-red & 0.769 & 0.00601 & -1.21 & \textbf{0.769} & 0.00203 & -3.96 & \textbf{0.769} & 0.00181 & -3.77 & \textbf{0.767} & 0.0019 & \textbf{-4.69} & 0.773 & 0.00186 & -1.7 & \textbf{0.769} & \textbf{0.00159} & \textbf{-1.19} & \textbf{0.769} & 0.00242 & \textbf{-1.22} \\
wine-white & 0.777 & 0.0015 & -1.22 & \sout{0.786} & \sout{0.000868} & \sout{-22.6} & \textbf{0.784} & 0.001 & -1.92 & \sout{0.784} & \sout{0.000767} & \sout{-7.37} & \sout{0.787} & \sout{0.000935} & \sout{-2.51} & \textbf{0.779} & 0.000613 & \textbf{-1.27} & \textbf{0.777} & \textbf{0.000514} & \textbf{-1.2} \\
yacht & 0.0226 & 0.882 & -0.919 & \sout{0.657} & \sout{0.0331} & \sout{-0.769} & \textbf{0.0226} & \textbf{0.00442} & \textbf{2.29} & \textbf{0.0226} & 0.00955 & 1.86 & \sout{0.135} & \sout{0.0103} & \sout{0.738} & \sout{0.454} & \sout{0.0876} & \sout{-0.0822} & \textbf{0.0226} & 0.0174 & \textbf{1.33} \\
\textit{{Total wins or ties}} & -- & -- & -- & \textit{1} & \textit{0} & \textit{0} & \textit{7} & \textit{2} & \textit{3} & \textit{6} & \textit{1} & \textit{3} & \textit{0} & \textit{0} & \textit{0} & \textit{2} & \textit{1} & \textit{2} & \textit{11} & \textit{8} & \textit{11} \\
\bottomrule
\end{tabular}
}
\end{table*}

We report mean estimation accuracy with RMSE.
For variance, we use expected calibration error (ECE)~\citep{kuleshov2018accurate}.
For bin edges $0 = p_0 < p_1 < \hdots <p_{m-1} \leq p_m = 1$, ECE computes bin probabilities
\begin{align*}
    \hat{p}_j = \frac{|\{(x,y)\in\D : p_{j-1} < F(y|x) \leq p_j\}|}{|\D|},
\end{align*}
where $F(y|x)$ is the Normal CDF.
For a well-calibrated model, $\hat{p}_j \approx p_j - p_{j-1}$.
Accordingly, ECE is defined as $\sum_{j=1}^{m} (\hat{p}_j - (p_j - p_{j-1}))^2$.
We report RMSE, ECE, and log likelihood (LL) in \cref{tab:uci} for all methods with Normal predictive distributions.
Our supplement has tables for Deep Ensemble, Monte Carlo Dropout, and Student based methods with similar findings.
Any model whose predictive squared errors are greater than those of the mean-only baseline (Unit Variance Homoscedastic) according to a one-sided paired t-test with a 0.05 threshold is considered empirically unfaithful and has its RMSE, ECE, and LL struck out.
We use a paired t-test since the squared errors from two models are not independent--each pair results from the same held out data point.
For each dataset, we bold the smallest RMSE and, using the same significance test, identify any statistical ties, which too are bold.
We also bold the smallest ECE, but ignore struck-out values from any unfaithful models.
We use a G-test to identify statistically significant differences between the histograms generated from a model's $\hat{p}_j$ values.
Any faithful model, whose $p$-value $\geq0.05$ when compared to the faithful model with the best ECE is considered tied for best.
Among faithful models, we bold the largest LL and any statistical ties according to a one-sided Kolmogorov-Smirnov test with a $0.05$ threshold.
Tallying the number of wins/ties for each metric, we find our Faithful Heteroscedastic model is never empirically unfaithful and offers the best combination of mean estimation accuracy, variance calibration, and data log likelihood.
Proposals 1 and 2 underperform in isolation, confirming both are necessary to achieve faithfulness.

\begin{table*}[t!]
    \caption{VAE Performance of Methods with Predictive Normal Distribtuions}
    \label{tab:vae}
    \adjustbox{width=\textwidth}{\begin{tabular}{ll|ccc|ccc|ccc|ccc|ccc|ccc|ccc}
\toprule
 &  & \multicolumn{3}{|c}{Unit Variance Homoscedastic} & \multicolumn{3}{|c}{Conventional Heteroscedastic} & \multicolumn{3}{|c}{Beta NLL (0.5)} & \multicolumn{3}{|c}{Beta NLL (1.0)} & \multicolumn{3}{|c}{Proposal 1} & \multicolumn{3}{|c}{Proposal 2} & \multicolumn{3}{|c}{Faithful Heteroscedastic} \\
 &  & \multicolumn{3}{|c}{(mean-only per \cref{def:mean-only})} & \multicolumn{3}{|c}{\citep{nix1994estimating}} & \multicolumn{3}{|c}{\citep{seitzer2022on}} & \multicolumn{3}{|c}{\citep{seitzer2022on}} & \multicolumn{3}{|c}{(Ours)} & \multicolumn{3}{|c}{(Ours)} & \multicolumn{3}{|c}{(Ours, Proposals 1 + 2)} \\
Dataset &  & RMSE & ECE & LL & RMSE & ECE & LL & RMSE & ECE & LL & RMSE & ECE & LL & RMSE & ECE & LL & RMSE & ECE & LL & RMSE & ECE & LL \\
\midrule
\multirow[t]{2}{*}{fashion-mnist} & clean & 0.923 & 0.00232 & -1.05e+03 & \sout{1.61} & \sout{0.00013} & \sout{-711} & \sout{0.962} & \sout{0.000161} & \sout{78.6} & \sout{0.945} & \sout{0.000374} & \sout{-452} & \sout{0.935} & \sout{2.26e-05} & \sout{-144} & \sout{1.52} & \sout{0.000322} & \sout{-477} & \textbf{0.923} & \textbf{7.12e-05} & \textbf{-217} \\
 & corrupt & 1.17 & 0.00251 & -1.26e+03 & \sout{1.77} & \sout{0.000178} & \sout{-482} & \sout{1.2} & \sout{0.000296} & \sout{-81.9} & \sout{1.19} & \sout{0.000576} & \sout{-976} & \sout{1.19} & \sout{2.4e-05} & \sout{-320} & \sout{1.64} & \sout{0.00042} & \sout{-370} & \textbf{1.17} & \textbf{6.33e-05} & \textbf{-392} \\
\multirow[t]{2}{*}{mnist} & clean & 0.759 & 0.00923 & -946 & \sout{2.27} & \sout{0.00072} & \sout{117} & \sout{0.85} & \sout{0.000182} & \sout{-3.56e+03} & \sout{0.78} & \sout{0.00106} & \sout{-741} & \sout{0.792} & \sout{1.89e-05} & \sout{546} & \sout{2.14} & \sout{0.00112} & \sout{-565} & \textbf{0.759} & \textbf{3.16e-05} & \textbf{119} \\
 & corrupt & 1.04 & 0.00839 & -1.15e+03 & \sout{2.33} & \sout{0.000555} & \sout{-361} & \sout{1.07} & \sout{0.000243} & \sout{-1.53e+04} & \sout{1.07} & \sout{0.00668} & \sout{-1.43e+05} & \sout{1.06} & \sout{2.51e-05} & \sout{220} & \sout{1.8} & \sout{0.00236} & \sout{-413} & \textbf{1.04} & \textbf{5.78e-05} & \textbf{-358} \\
\textit{{Total wins or ties}} &  & -- & -- & -- & \textit{0} & \textit{0} & \textit{0} & \textit{0} & \textit{0} & \textit{0} & \textit{0} & \textit{0} & \textit{0} & \textit{0} & \textit{0} & \textit{0} & \textit{0} & \textit{0} & \textit{0} & \textit{4} & \textit{4} & \textit{4} \\
\bottomrule
\end{tabular}
}
\end{table*}

\subsection{Decomposing Variance Estimates}\label{subsec:vae}

\Citet{kendall2017uncertainties,lakshminarayanan2017simple} want predictive variance to include epistemic and aleatoric uncertainty.
Conversely, we seek to accurately decompose predictive variance into its epistemic and aleatoric components.
Consider a variational autoencoder (VAE)~\citep{kingma2013auto} that stochastically encodes an image into a low-dimensional space and employs a Gaussian decoder.
A heteroscedastic Gaussian decoder will use per-pixel predictive variance to explain reconstruction errors.
Collecting more data is unlikely to reduce these variances.
Yet, we could replace the VAE with a deterministic autoencoder to trivially achieve infinite likelihood: use the identity function to autoencode the mean and set variance to zero.
Thus, we argue a VAE's variance estimates for a set of unique images are entirely epistemic (i.e.~only reducible with a model change) and capture compression loss.

If we define $f_\text{trunk}$ as a stochastic, convolutional encoder $q(z;x)$, then \cref{fig:comp-graph} becomes a VAE.
A VAE minimizes
\begin{align*}
    \E_{q(z;x)}[-\log \N(x;f_\mu(z), f_\Sigma(z))] + D_{KL}\big(q(z;x) || p(z)\big).
\end{align*}
The KL divergence is simply a regularization loss for $f_\text{trunk}$.
Using a single sample from $q(z;x)$, we use our chosen baselines' proposed objective for the NLL.
We note, however, our loss is equivalent to optimization changes such that we are not changing the VAE model but rather its optimization.
Parameter heads $f_\mu$ and $f_\Sigma$ are each transposes of the encoder.
Our supplement contains additional details.

\begin{figure}[t]
    \includegraphics[width=0.475\textwidth]{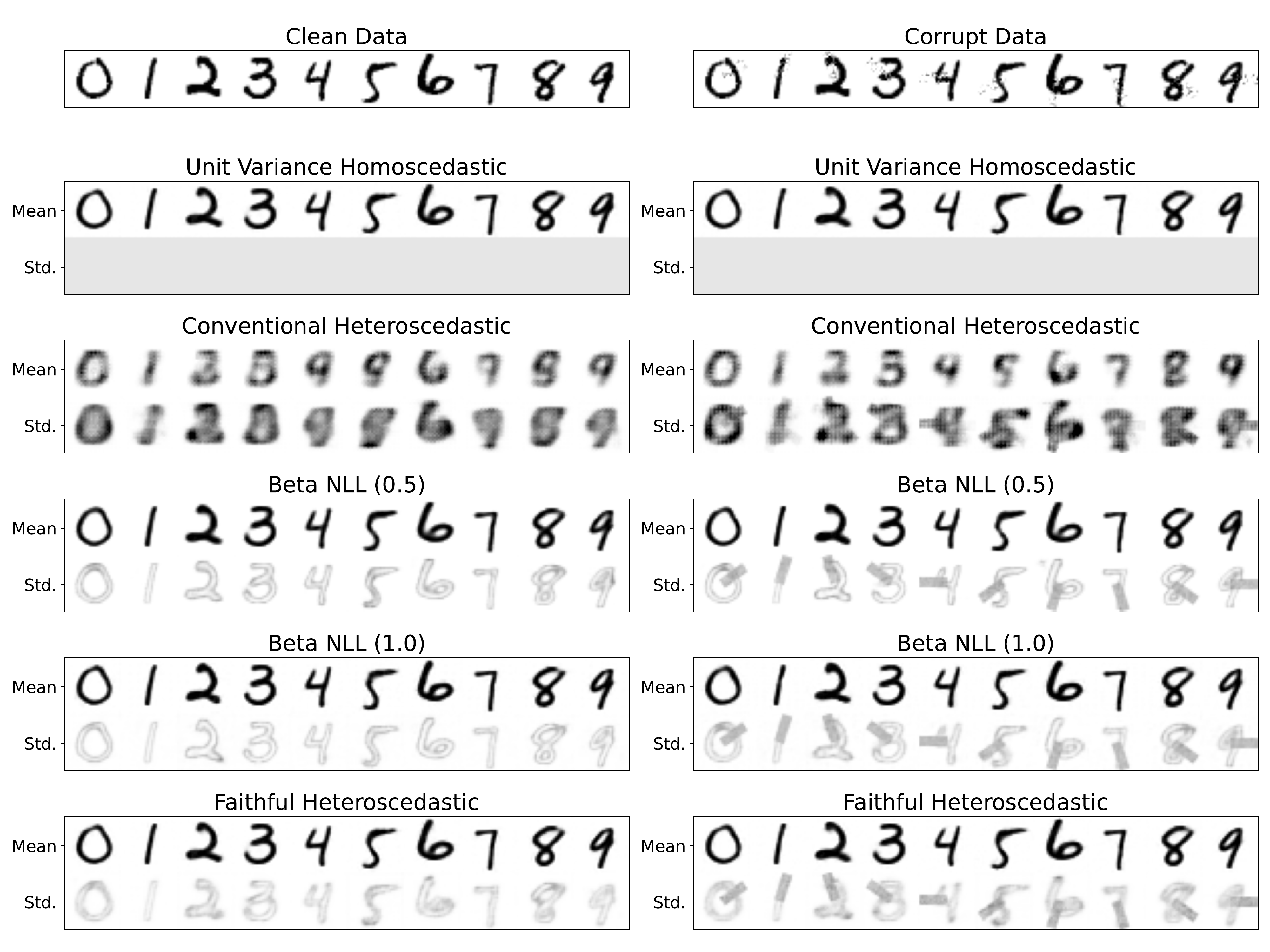}
    \caption{VAE Predictive Moments for MNIST}
    \label{fig:vae-moments}
\end{figure}

The left column of \cref{fig:vae-moments} shows an example set of images and each model's corresponding mean and standard deviation estimates.
The Conventional Heteroscedastic model uses high variance to explain poor mean fits.
All other models produce sensible looking means.
Their estimated variances suggest compression loss most significantly affects edge localization, which we argue is epistemic.

We now add simulated heteroscedastic noise to each image, where the standard deviation is a rectangular patch whose magnitude is 25\% the dynamic range of the images.
The patch is rotated according to an image's class label (\cref{fig:vae-noise}, top) to simulate heteroscedastic aleatoric uncertainty.
The right column of \cref{fig:vae-moments} shows an example set of corrupt images and each model's moment estimates when trained on the corrupt data.
Except for the Conventional Heteroscedastic model, the variance estimates contain the same epistemic uncertainty from the clean data as well as the added aleatoric noise variance.
To isolate heteroscedastic aleatoric uncertainty for an image, we subtract the variance output from a model trained on clean data from that of a model trained on corrupt data.
The well-known Bias-Variance equality supports this decomposition:
\begin{align*}
    \footnotesize
    \underbrace{\text{Var}[y|x]}_\text{noise variance}
    &=
    \underbrace{\E[(y-\mu(x))^2]}_\text{aleatoric and epistemic} -
    \underbrace{(\E[y|x] - \mu(x))^2}_\text{only epistemic mean uncertainty} \\
    &\approx
    \underbrace{\Sigma_\text{noisy}(x)}_\text{model trained w/ noisy $y|x$} -
    \underbrace{\Sigma_\text{clean}(x)}_\text{model trained w/ $\E[y|x]$}.
\end{align*}
We plot the average difference in predictive variance ($\Sigma_\text{noisy}(x) - \Sigma_\text{clean}(x)$) over all images within a class in \cref{fig:vae-noise}.
The ability to recover the true noise variance suggests that, for datasets where we have clean and noisy measurements, we can use well-calibrated heteroscedastic models to decompose the predicted variance into its aleatoric and epistemic components.
Our supplement includes full size images for all datasets.

\Cref{tab:vae} reports performance and identifies wins/ties as described in \cref{subsec:uci}.
Again, our Faithful Heteroscedastic model offers the best combined performance.
All other models were empirically unfaithful.

\begin{figure}[t]
    \centering
    \includegraphics[width=0.475\textwidth]{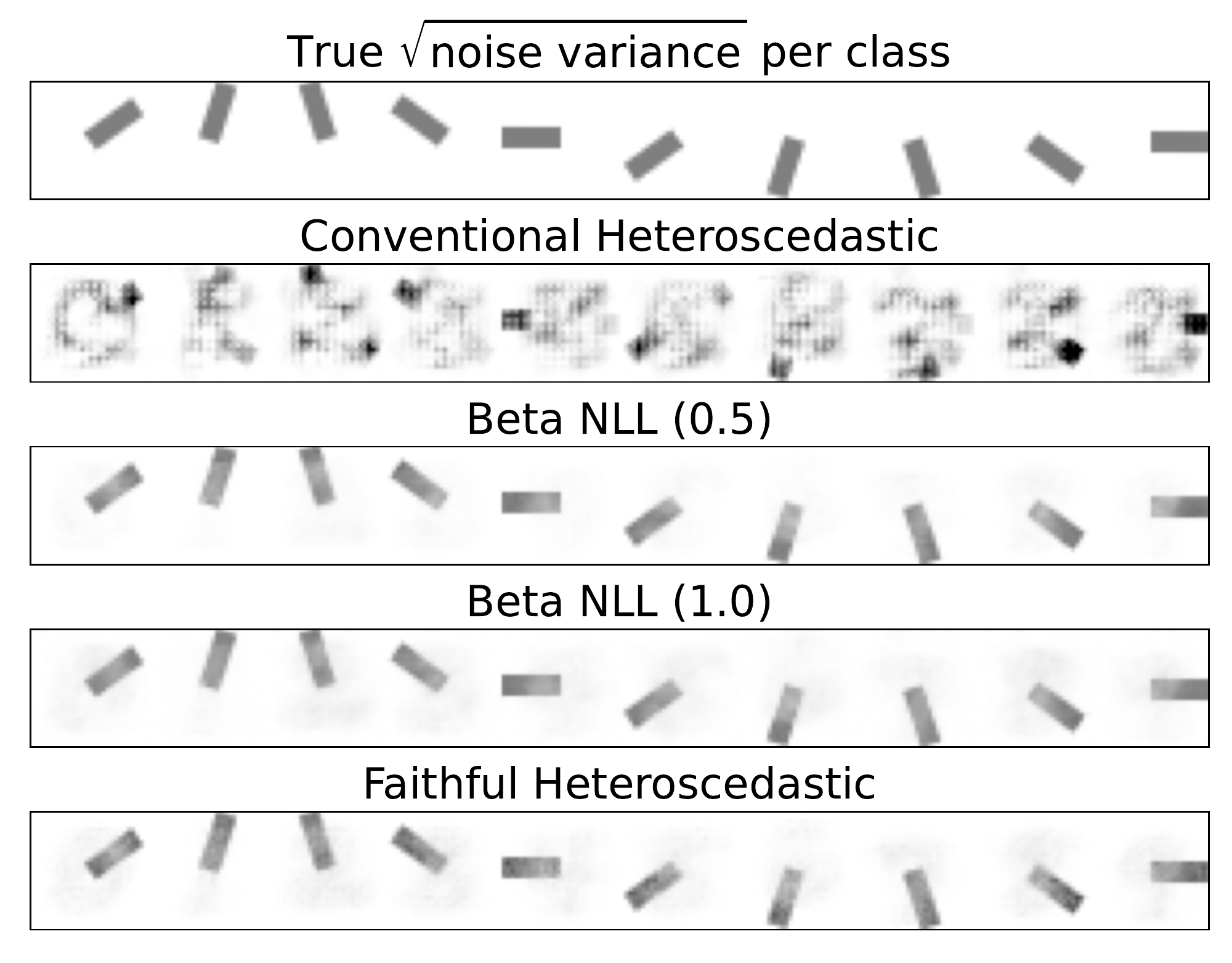}
    \caption{VAE Recovered Noise Variance for MNIST}
    \label{fig:vae-noise}
\end{figure}

\subsection{Noise Variance in CRISPR-Cas13 efficacy}\label{subsec:crispr}

\begin{table*}[t!]
    \caption{Modeling CRISPR Cas13 Efficacy}
    \label{tab:crispr}
    \adjustbox{width=\textwidth}{\begin{tabular}{ll|ccc|ccc|ccc|ccc|ccc|ccc|ccc}
\toprule
 &  & \multicolumn{3}{|c}{Unit Variance Homoscedastic} & \multicolumn{3}{|c}{Conventional Heteroscedastic} & \multicolumn{3}{|c}{Beta NLL (0.5)} & \multicolumn{3}{|c}{Beta NLL (1.0)} & \multicolumn{3}{|c}{Proposal 1} & \multicolumn{3}{|c}{Proposal 2} & \multicolumn{3}{|c}{Faithful Heteroscedastic} \\
 &  & \multicolumn{3}{|c}{(mean-only per \cref{def:mean-only})} & \multicolumn{3}{|c}{\citep{nix1994estimating}} & \multicolumn{3}{|c}{\citep{seitzer2022on}} & \multicolumn{3}{|c}{\citep{seitzer2022on}} & \multicolumn{3}{|c}{(Ours)} & \multicolumn{3}{|c}{(Ours)} & \multicolumn{3}{|c}{(Ours, Proposals 1 + 2)} \\
Dataset & Observations & RMSE & ECE & LL & RMSE & ECE & LL & RMSE & ECE & LL & RMSE & ECE & LL & RMSE & ECE & LL & RMSE & ECE & LL & RMSE & ECE & LL \\
\midrule
\multirow[t]{2}{*}{flow cytometry (HEK293)} & means & 0.279 & 0.0209 & -0.958 & \sout{0.285} & \sout{0.0925} & \sout{-11.4} & \textbf{0.272} & \textbf{0.0214} & -1.27 & \textbf{0.274} & 0.0258 & -1.25 & \sout{0.297} & \sout{0.0088} & \sout{-0.915} & \sout{0.289} & \sout{0.0599} & \sout{-5.21} & 0.279 & 0.0219 & \textbf{-1.04} \\
 & replicates & 0.294 & 0.00673 & -0.962 & 0.295 & 0.0117 & -1.2 & \textbf{0.286} & 0.00697 & -0.658 & 0.294 & 0.00695 & -0.65 & \sout{0.304} & \sout{0.00925} & \sout{-1.12} & \sout{0.301} & \sout{0.00635} & \sout{-0.697} & 0.294 & \textbf{0.00459} & \textbf{-0.511} \\
\multirow[t]{2}{*}{survival screen (A375)} & means & 0.31 & 0.00596 & -0.967 & \sout{0.313} & \sout{0.000577} & \sout{-0.251} & \textbf{0.309} & 0.00131 & -0.337 & \textbf{0.309} & 0.00066 & -0.262 & \sout{0.314} & \sout{0.000948} & \sout{-0.328} & \sout{0.312} & \sout{0.000578} & \sout{-0.246} & \textbf{0.31} & \textbf{0.000599} & \textbf{-0.248} \\
 & replicates & 0.36 & 0.00158 & -0.984 & \sout{0.361} & \sout{0.000324} & \sout{-0.422} & \textbf{0.36} & 0.000311 & -0.422 & \textbf{0.36} & 0.000243 & -0.406 & \sout{0.363} & \sout{0.00033} & \sout{-0.437} & \sout{0.361} & \sout{0.000183} & \sout{-0.391} & \textbf{0.36} & \textbf{0.000206} & \textbf{-0.394} \\
\multirow[t]{2}{*}{survival screen (HEK293) } & means & 0.304 & 0.0178 & -0.965 & \sout{0.309} & \sout{0.0111} & \sout{-1.92} & \textbf{0.298} & 0.00362 & -0.261 & 0.301 & 0.00271 & -0.211 & \sout{0.308} & \sout{0.0033} & \sout{-0.282} & \sout{0.318} & \sout{0.00433} & \sout{-0.439} & 0.304 & \textbf{0.00235} & \textbf{-0.194} \\
 & replicates & 0.346 & 0.00281 & -0.979 & \sout{0.349} & \sout{0.00128} & \sout{-0.45} & \textbf{0.345} & 0.000663 & -0.337 & \sout{0.348} & \sout{0.000551} & \sout{-0.332} & \sout{0.352} & \sout{0.000859} & \sout{-0.403} & \sout{0.351} & \sout{0.000869} & \sout{-0.381} & \textbf{0.346} & \textbf{0.000453} & \textbf{-0.316} \\
\textit{{Total wins or ties}} &  & -- & -- & -- & \textit{0} & \textit{0} & \textit{0} & \textit{6} & \textit{1} & \textit{0} & \textit{3} & \textit{0} & \textit{0} & \textit{0} & \textit{0} & \textit{0} & \textit{0} & \textit{0} & \textit{0} & \textit{3} & \textit{5} & \textit{6} \\
\bottomrule
\end{tabular}
}
\end{table*}

While our example of corrupting a clean set of images with heteroscedastic noise in \cref{subsec:vae} may seem a bit contrived, a similar situation arises in computational biology all the time.
In particular, biologists will often replicate experiments to account for technical variation.
Having access to a dataset with the raw replicates allows for the creation of a dataset comprised of averaged (de-noised) replicates.

CRISPR-Cas13 is a system for knocking down gene expression.
Unlike Cas9, which modifies DNA, Cas13 targets and degrades specific RNA transcripts.
Cas13 systems use a guide RNA (gRNA) to recruit the Cas13 protein to destroy complementary RNA transcripts in its host.
Cas13 has known sequence preferences--some sequences are more targetable than others.
We have access to three datasets measuring Cas13 efficacy that vary in cell and screen type.
\ifnum\statePaper=0{
While one dataset~\citep{wessels_massively_2020} is publicly available, the other two are subsets of data currently in review at biology venues\footnote{We promise to release these data by camera-ready. We certify the submissions to biology venues do not violate AISTATS' dual submission policy.}.
}\else{
The flow-cytometry dataset~\citep{wessels_massively_2020} is publicly available.
The other two will become publicly available in 2023.
}\fi
For each dataset, we have multiple efficacy scores for each target sequence as a result of the replicated experiments.
We predict a scalar value that measures knockdown efficacy (a more negative value indicates higher efficacy) from one-hot encoded target sequence.
Shared representation learner $f_\text{trunk}$ uses convolutional layers for motif discovery.
The parameter heads are fully-connected networks.
See supplement for details.

SHAP values quantify each feature's impact to a model's output w.r.t.~the average output taken over a set of background samples~\citep{lundberg_unified_2017}.
Using the same cross-validation strategy from \cref{subsec:uci}, we collect the predictive moments and their SHAP values for every data point;
we do this for all models twice: once when trained on raw replicates and again when trained on their average.
\Cref{tab:crispr} reports performance, identifying wins/ties as described in \cref{subsec:uci}.
While our Faithful Heteroscedastic model does not always have the best RMSE, it still offers the best ECE and combined performance.

\begin{figure}[h!]
    \centering
    \includegraphics[width=0.475\textwidth]{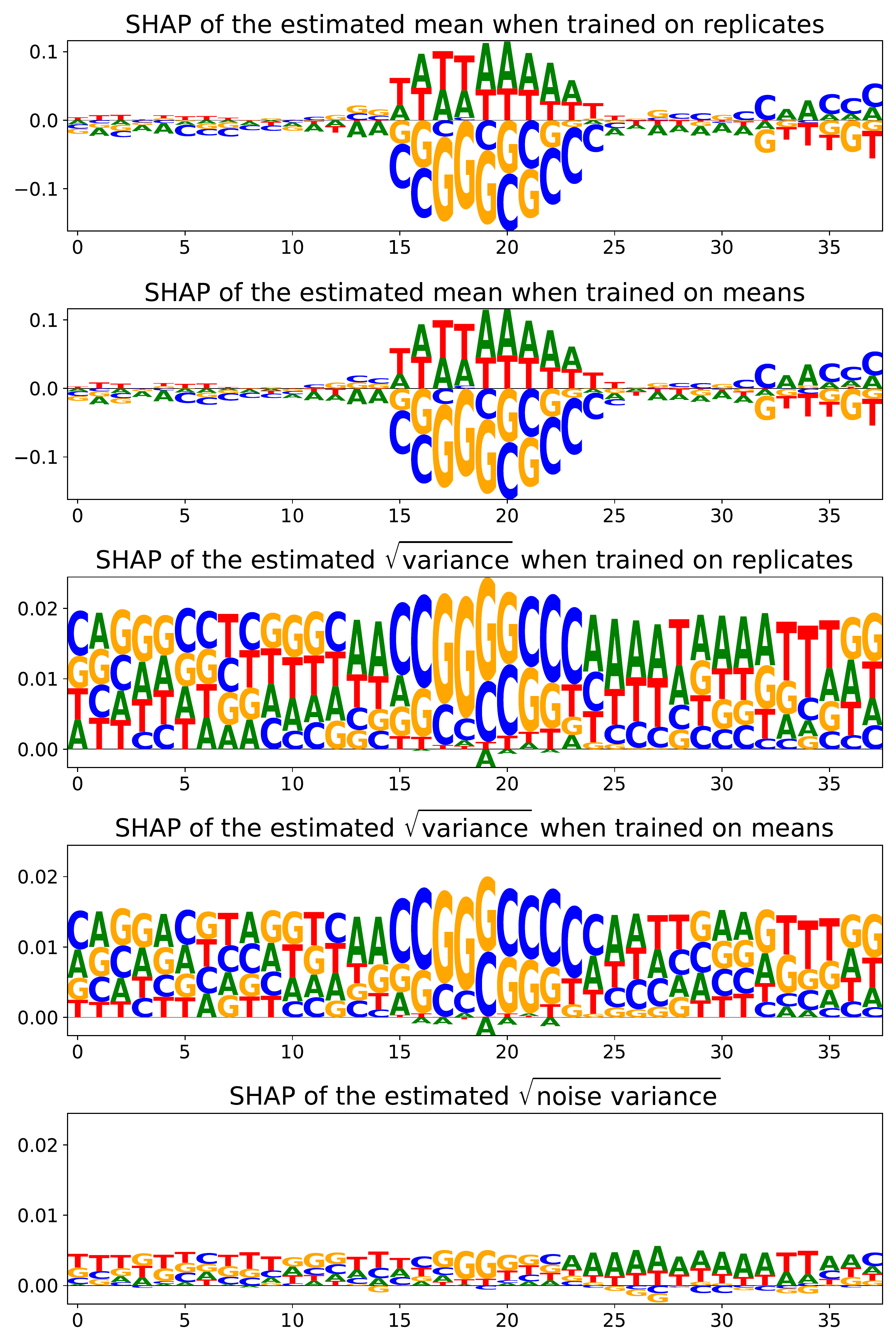}
    \caption{SHAP Analysis of Flow Cytometry (HEK293) Data Using our Proposed Faithful Heteroscedastic Model}
    \label{fig:crispr}
\end{figure}

Regardless of whether trained on means or replicates, our model's SHAP values for its mean output recover known Cas13 sequence preferences (\cref{fig:crispr}, first two rows) in HEK293 cells as measured by flow cytometry.
The next two rows of \cref{fig:crispr}, plot the average SHAP values for our model's variance output.
As expected, the magnitude decreases for the model trained on the averaged efficacy replicates.
The last row of \cref{fig:crispr}, subtracts the variance's SHAP values for our model when trained on averaged efficacy scores from those trained on raw scores (analogous to \cref{subsec:vae}).
This results shows that heteroscedasticity in CRISPR-Cas13 has a similar sequence dependence as the mean, suggesting more effective guides experience more measurement noise--a common heteroscedastic outcome.
To the best of our knowledge, we are the first to model sequence-dependent heteroscedasticity in CRISPR-Cas13.
Our supplement provides similar plots for the other datasets and models.

\section{DISCUSSION}\label{sec:discussion}

If the parameterizing neural network for a heteroscedastic Gaussian model has all components removed that are not necessary for generating mean estimates, the resulting subnetwork will have the same expressive power for mean estimation as the full model.
Therefore, we posit any heteroscedastic model's mean estimation accuracy should be as good as this mean-only baseline;
we say a model and its optimization are `faithful' upon achieving this.
Yet, conventional minimization of a heteroscedastic model's NLL can worsen its mean estimation accuracy~\citep{seitzer2022on}.
Our optimization proposals provably ensure `faithfulness' and are equivalent to a modified loss function, making our method exceptionally easy to implement in popular deep learning frameworks.
One limitation of our approach, a byproduct of \cref{thm:faithful}, is that our method can never have more accurate mean estimates than its mean-only baseline.
There are a few examples of the Beta NLL model~\citep{seitzer2022on} outperforming its mean-only baseline's mean estimation accuracy.

Our experiments include datasets ranging from simple scalar regression to auto-encoding natural images and predicting CRISPR-Cas13 efficacy from RNA sequence.
The tested network architectures are equally diverse.
Regardless of task or network complexity, our method overwhelmingly offers the best combination of mean estimation accuracy and variance calibration.
With our reliable mean and variance estimates, we demonstrate how to isolate aleatoric uncertainty in the predictive variance from epistemic uncertainty.
Our technique accurately recovers heteroscedasticity in simulated experiments.
Additionally, our model recovers known CRISPR-Cas13 sequence preferences and identifies low-levels of sequence-dependent (heteroscedastic) noise variance.
When model predictions inform high-impact decision-making (e.g.~those that affect health or society), accurate mean and uncertainty estimates are critical.
Thus, our method represents a step towards improving model trustworthiness and reliability.

\subsubsection*{Acknowledgements}
A.S.~devised the ideas in this article and implemented all computational experiments.
H.W.~collected and prepared both HEK293 datasets.
M.S. \& L.P.~collected and prepared the A735 dataset.
N.S.~supervises H.W.
D.A.~supervises A.S. \& M.S.

%
%
%
%

\bibliography{references}

\begin{thebibliography}{}

\bibitem[Chua et~al., 2018]{chua2018deep}
Chua, K., Calandra, R., McAllister, R., and Levine, S. (2018).
\newblock Deep reinforcement learning in a handful of trials using
  probabilistic dynamics models.
\newblock {\em Advances in neural information processing systems}, 31.

\bibitem[Der~Kiureghian and Ditlevsen, 2009]{der2009aleatory}
Der~Kiureghian, A. and Ditlevsen, O. (2009).
\newblock Aleatory or epistemic? does it matter?
\newblock {\em Structural safety}, 31(2):105--112.

\bibitem[Gal and Ghahramani, 2016]{gal2016dropout}
Gal, Y. and Ghahramani, Z. (2016).
\newblock Dropout as a bayesian approximation: Representing model uncertainty
  in deep learning.
\newblock In {\em international conference on machine learning}, pages
  1050--1059. PMLR.

\bibitem[Gal et~al., 2017]{gal2017deep}
Gal, Y., Islam, R., and Ghahramani, Z. (2017).
\newblock Deep bayesian active learning with image data.
\newblock In {\em International Conference on Machine Learning}, pages
  1183--1192. PMLR.

\bibitem[Goldberg et~al., 1997]{goldberg1997regression}
Goldberg, P., Williams, C., and Bishop, C. (1997).
\newblock Regression with input-dependent noise: A gaussian process treatment.
\newblock {\em Advances in neural information processing systems}, 10.

\bibitem[H{\"o}fling and Tibshirani, 2008]{hofling2008study}
H{\"o}fling, H. and Tibshirani, R. (2008).
\newblock A study of pre-validation.
\newblock {\em The Annals of Applied Statistics}, 2(2):643--664.

\bibitem[Kendall and Gal, 2017]{kendall2017uncertainties}
Kendall, A. and Gal, Y. (2017).
\newblock What uncertainties do we need in bayesian deep learning for computer
  vision?
\newblock {\em Advances in neural information processing systems}, 30.

\bibitem[Kingma and Ba, 2014]{kingma2014adam}
Kingma, D.~P. and Ba, J. (2014).
\newblock Adam: A method for stochastic optimization.
\newblock {\em arXiv preprint arXiv:1412.6980}.

\bibitem[Kingma and Welling, 2013]{kingma2013auto}
Kingma, D.~P. and Welling, M. (2013).
\newblock Auto-encoding variational bayes.
\newblock {\em arXiv preprint arXiv:1312.6114}.

\bibitem[Kuleshov et~al., 2018]{kuleshov2018accurate}
Kuleshov, V., Fenner, N., and Ermon, S. (2018).
\newblock Accurate uncertainties for deep learning using calibrated regression.
\newblock In {\em International conference on machine learning}, pages
  2796--2804. PMLR.

\bibitem[Lakshminarayanan et~al., 2017]{lakshminarayanan2017simple}
Lakshminarayanan, B., Pritzel, A., and Blundell, C. (2017).
\newblock Simple and scalable predictive uncertainty estimation using deep
  ensembles.
\newblock {\em Advances in neural information processing systems}, 30.

\bibitem[Lundberg and Lee, 2017]{lundberg_unified_2017}
Lundberg, S.~M. and Lee, S.-I. (2017).
\newblock A unified approach to interpreting model predictions.
\newblock {\em Advances in neural information processing systems}, 30.

\bibitem[Nix and Weigend, 1994]{nix1994estimating}
Nix, D.~A. and Weigend, A.~S. (1994).
\newblock Estimating the mean and variance of the target probability
  distribution.
\newblock In {\em Proceedings of 1994 ieee international conference on neural
  networks (ICNN'94)}, volume~1, pages 55--60. IEEE.

\bibitem[Osband et~al., 2016]{osband2016deep}
Osband, I., Blundell, C., Pritzel, A., and Van~Roy, B. (2016).
\newblock Deep exploration via bootstrapped dqn.
\newblock {\em Advances in neural information processing systems}, 29.

\bibitem[Rasmussen and Williams, 2006]{rasmussen_gaussian_2006}
Rasmussen, C.~E. and Williams, C. K.~I. (2006).
\newblock {\em Gaussian processes for machine learning}.
\newblock Adaptive computation and machine learning. MIT Press, Cambridge,
  Mass.
\newblock OCLC: ocm61285753.

\bibitem[Seitzer et~al., 2022]{seitzer2022on}
Seitzer, M., Tavakoli, A., Antic, D., and Martius, G. (2022).
\newblock On the pitfalls of heteroscedastic uncertainty estimation with
  probabilistic neural networks.
\newblock In {\em International Conference on Learning Representations}.

\bibitem[Skafte et~al., 2019]{skafte2019reliable}
Skafte, N., J{\o}rgensen, M., and Hauberg, S. (2019).
\newblock Reliable training and estimation of variance networks.
\newblock {\em Advances in Neural Information Processing Systems}, 32.

\bibitem[Stirn and Knowles, 2020]{stirn2020variational}
Stirn, A. and Knowles, D.~A. (2020).
\newblock Variational variance: Simple and reliable predictive variance
  parameterization.
\newblock {\em arXiv preprint arXiv:2006.04910}.

\bibitem[Takahashi et~al., 2018]{takahashi2018student}
Takahashi, H., Iwata, T., Yamanaka, Y., Yamada, M., and Yagi, S. (2018).
\newblock Student-t variational autoencoder for robust density estimation.
\newblock In {\em IJCAI}, pages 2696--2702.

\bibitem[Wessels et~al., 2020]{wessels_massively_2020}
Wessels, H.-H., Méndez-Mancilla, A., Guo, X., Legut, M., Daniloski, Z., and
  Sanjana, N.~E. (2020).
\newblock Massively parallel {Cas13} screens reveal principles for guide {RNA}
  design.
\newblock {\em Nature Biotechnology}, 38(6):722--727.

\bibitem[Wilson et~al., 2016]{wilson2016deep}
Wilson, A.~G., Hu, Z., Salakhutdinov, R., and Xing, E.~P. (2016).
\newblock Deep kernel learning.
\newblock In {\em Artificial intelligence and statistics}, pages 370--378.
  PMLR.

\bibitem[Yu et~al., 2020]{yu2020mopo}
Yu, T., Thomas, G., Yu, L., Ermon, S., Zou, J.~Y., Levine, S., Finn, C., and
  Ma, T. (2020).
\newblock Mopo: Model-based offline policy optimization.
\newblock {\em Advances in Neural Information Processing Systems},
  33:14129--14142.

\end{thebibliography}
\bibliographystyle{apalike}

\ifnum\stateSupp=1{
\clearpage
\onecolumn
\appendix

\section{FURTHER DISCUSSION OF RELATED WORK}\label{sec:further-discussion-of-related-work}

\citet{skafte2019reliable} propose four modifications to improve variance estimates.
While their combination yields a predictive Student distribution, one of their proposals on its own admits a Normal model.
Utilizing separate mean and variance networks, they use the first half of training to fit just the mean via minimizing sum of squared errors (SSE).
Thereafter, they alternate between minimizing the negative log likelihood (NLL) w.r.t.~mean network parameters and minimizing the NLL w.r.t.~variance network parameters.
The exact number of batches between alternations is a configurable hyperparameter.
They argue that fitting the variance should wait until mean estimates are reasonable.
This proposal is similar to the faithful approach discussed in section 2 of our main manuscript--use separate mean and variance networks, minimize the SSE of the mean network, freeze its parameters, and then minimize the NLL of the variance network--and suffers the same deficiencies.
Also, without freezing mean network parameters before fitting the separate variance network, their approach is not provably faithful.
That said, if the mean network is sufficiently converged, then \citet{skafte2019reliable} can avoid the deleterious `rich-get-richer' cycle where poor mean estimates are explained by high variance and subsequently ignored~\citep{seitzer2022on}.
However, their implementation has the added complication of running two adaptive gradient optimizers--one for the mean network and one for the variance network.
The gradient w.r.t.~one network's parameters could be drastically different after an interleaved adjustment to the other network's parameters.
It is unclear what effect, if any, this might have on Adam's~\citep{kingma2014adam} performance.
\citet{stirn2020variational} find that \citet{skafte2019reliable}'s combination of proposals can severely overestimate predictive variance in high-dimensional settings.

As discussed in our main manuscript, \citet{seitzer2022on} propose their $\beta$-NLL loss
\begin{align*}
    \LL = \sum_{(x,y)\in\D} \lfloor \sigma^{2\beta}(x) \rfloor \Bigg(\frac{1}{2}\log \sigma^2(x) + \frac{(y - \mu(x))^2}{2\sigma^2(x)}\Bigg).
\end{align*}
When $\beta=0$, their objective is the same as the NLL.
For $0 < \beta \leq 1$, the gradients become
\begin{align*}
    \nabla_{\mu(x)} \LL = \sum_{(x,y)\in\D} \frac{\mu(x) - y}{\sigma^{(2-2\beta)}(x)}, \quad
    \nabla_{\sigma^2(x)} \LL = \sum_{(x,y)\in\D} \frac{\sigma^2(x) - (y - \mu(x))^2}{2\sigma^{(4-2\beta)}(x)}.
\end{align*}
The gradient w.r.t.~the mean estimate is linear in $\mu(x)$ for all $\beta\in[0,1]$.
The gradient w.r.t.~the variance estimate is a difference of powers of $\sigma^2(x)$.
Changing these powers naturally affects the curvature of the gradient.
Supplement \cref{fig:beta-grad} plots $\nabla_{\sigma^2(x)} \LL$ when $(y - \mu(x))^2=1$ for different values of $\beta$.

\begin{figure}[h!]
    \centering
    \begin{adjustbox}{width=7cm}
    \begin{tikzpicture}
    \begin{axis}[axis lines=left, xlabel=\(\sigma^2(x)\), ylabel={\(\nabla_{\sigma^2(x)} \LL\)}, legend pos=outer north east]
        \addplot[domain=0.4:5, samples=1000, color=blue] {(x-1) / (2 * x^2)};
        \addlegendentry{$\beta=0.0$}
        \addplot[domain=0.4:5, samples=1000, color=orange]  {(x-1) / (2 * x^(3/2))};
        \addlegendentry{$\beta=0.5$}
        \addplot[domain=0.4:5, samples=1000, color=green] {(x-1) / (2 * x)};
        \addlegendentry{$\beta=1.0$}
    \end{axis}
    \end{tikzpicture}
    \end{adjustbox}
    \caption{Gradient of $\beta$-NLL w.r.t. Variance for Unit Mean Error at Different $\beta$ Settings}\label{fig:beta-grad}
\end{figure}
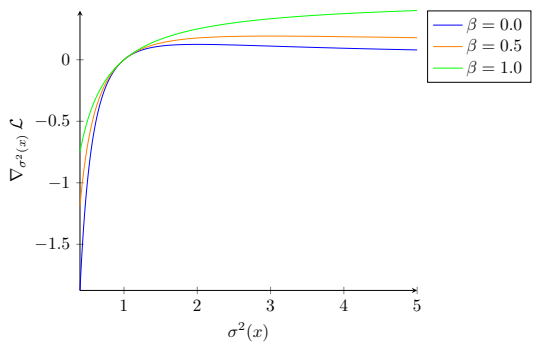

While the optima at $\sigma^2(x)=1$ does not move, the magnitude of the gradient about the optima does.
To the left of the optima, the magnitude of the gradient decreases as $\beta$ increases.
The opposite effect occurs to the right of the optima--magnitude of the gradient increases as $\beta$ increases.
Also, at $\beta=1$, $\nabla_{\sigma^2(x)} \LL$ becomes fully concave.
Given these changes, it is important to test their predictive variance calibration.

Our first proposal scales the gradient w.r.t.~the mean estimate by its inverse Jacobian and is identical to using $\beta=1$ for $\nabla_{\mu(x)} \LL$.
Our proposals do not alter the NLL's gradient w.r.t.~the variance estimate and thus implicitly set $\beta=0$ for $\nabla_{\sigma^2(x)} \LL$.
Experimentally, we consistently offer better variance calibration than \citet{seitzer2022on}.

\section{ADAPTING TO OTHER MODEL CLASSES}\label{sec:adaption}
As stated in the `Related Solutions' portion of our main manuscript (section 2), we only compare performances within model classes to avoid confounding model selection and optimization methods.
We generally define a model class by the resulting predictive distribution.
Our main article only considers Normal predictive distributions.
The following subsections adapt our proposals to different types of predictive distributions and compare to relevant methods.
Upcoming supplement \cref{sec:experiment-details} contains the architectural details for the models used in section 3 of our main manuscript.
Unless otherwise noted the following additional model classes use the same architectures.

\subsection{Adaptation to Deep Ensemble Methods}\label{sec:adaptation-de}

\citet{lakshminarayanan2017simple} use an ensemble of heteroscedastic Gaussian models with neural network parameter maps to improve predictive uncertainty estimation.
Their predictive distribution is a uniform mixture of $M$ models
\begin{align*}
    p(Y|X=x) = \frac{1}{M}\sum_{m=1}^M \N(Y | \mu(x,\theta_m), \Sigma(x,\theta_m)),
\end{align*}
where the component models are separately initialized and trained via conventional NLL minimization.
To construct a `faithful' ensemble, we simply train the $M$ components with our method to guarantee each component is `faithful.'
The Unit Variance Homoscedastic is a uniform mixture of $M$ mean-only models:
\begin{align*}
    p(Y|X=x) = \frac{1}{M}\sum_{m=1}^M \N(Y | \mu(x,\theta_m), I_{\dim(Y)}).
\end{align*}

\begin{figure}[ht!]
    \centering
    \includegraphics[width=0.75\textwidth]{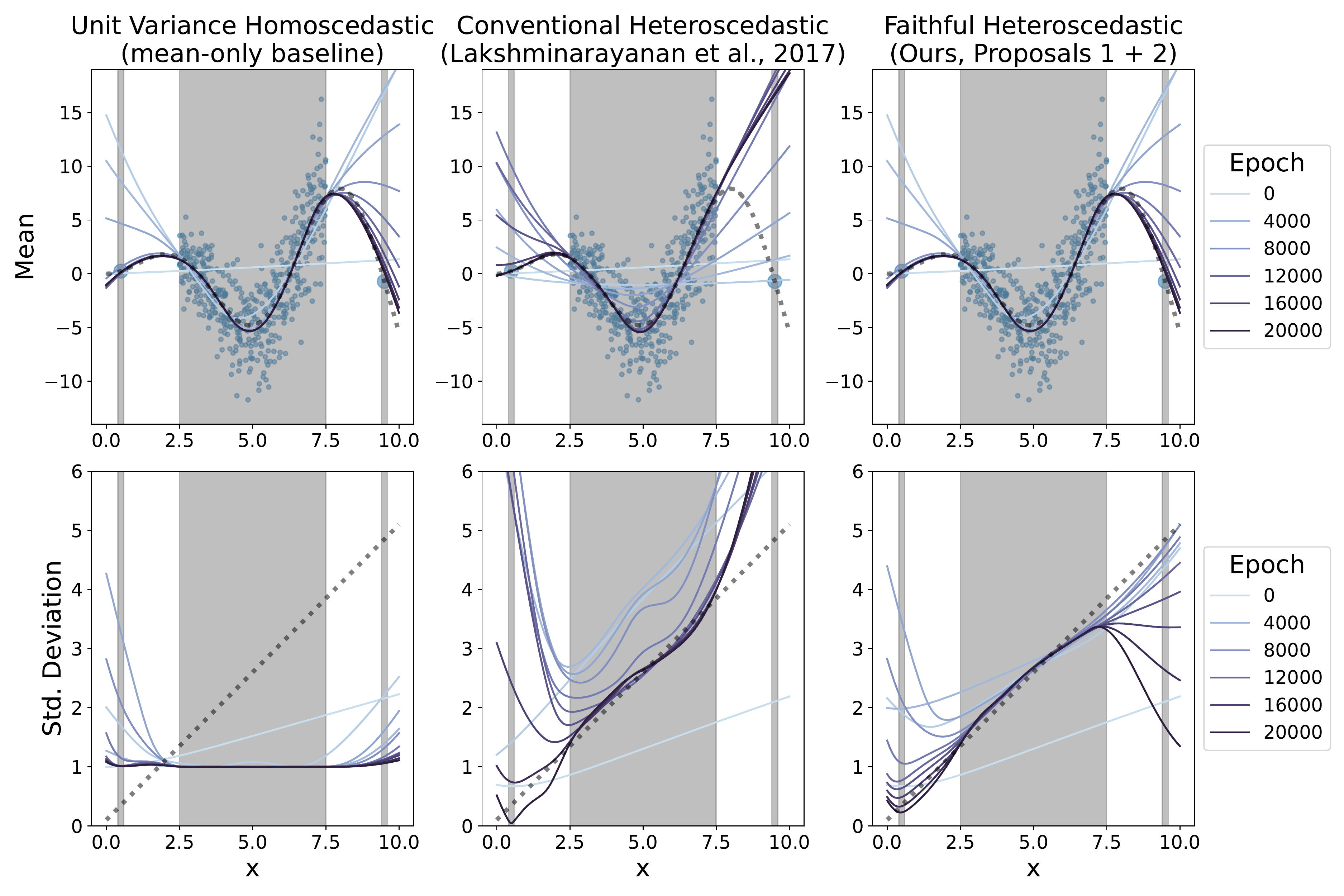}
    \caption{Convergence Behavior of Deep Ensemble Methods}
    \label{fig:covergence-de}
\end{figure}

\Cref{fig:covergence-de} examines the convergence behavior of Deep Ensemble methods.
Given the convergence flaws of the underlying component model (our main report; \citet{seitzer2022on}), it is unsurprising that an ensemble of these flawed models does not perform well.
Indeed, the Conventional Heteroscedastic Deep Ensemble~\citep{lakshminarayanan2017simple} uses high variance to explain mean errors at $X=9.5$ despite having the flexibility to converge thereto as demonstrated by
the Unit Variance Homoscedastic Deep Ensemble.
While the Unit Variance Homoscedastic Deep Ensemble uses a fixed variance for each of its components, the predictive variance of the ensemble changes w.r.t. $x$.
The increase in variance outside the data-rich interval $[2.5,7.5]$ arises from epistemic uncertainty in the $M$ mean estimates.
In particular, if any of the $\mu(x,\theta_m)$ components have slightly different outputs, then the variance of the predictive mixture increases.
\Cref{tab:uci-de} reports UCI regression performances for the Deep Ensemble methods.
Our faithful variant overwhelmingly outperforms the ensemble of conventionally optimized models.

\begin{table}[ht!]
    \centering
    \caption{UCI Regression Performance of Deep Ensemble Methods}
    \label{tab:uci-de}
    \adjustbox{width=0.75\textwidth}{\begin{tabular}{l|ccc|ccc|ccc}
\toprule
 & \multicolumn{3}{|c}{Unit Variance Homoscedastic} & \multicolumn{3}{|c}{Conventional Heteroscedastic} & \multicolumn{3}{|c}{Faithful Heteroscedastic} \\
Dataset & RMSE & ECE & LL & RMSE & ECE & LL & RMSE & ECE & LL \\
\midrule
boston & 0.32 & 0.172 & -0.975 & \sout{0.374} & \sout{0.00793} & \sout{-0.446} & \textbf{0.32} & \textbf{0.0053} & \textbf{-5.86} \\
carbon & 0.0487 & 0.48 & -2.76 & \textbf{0.0488} & 0.0169 & \textbf{9.35} & \textbf{0.0487} & \textbf{0.0111} & 4.99 \\
concrete & 0.25 & 0.154 & -0.955 & \sout{0.292} & \sout{0.00165} & \sout{-0.725} & \textbf{0.25} & \textbf{0.00374} & \textbf{-0.267} \\
energy & 0.159 & 0.327 & -1.85 & \sout{0.192} & \sout{0.0331} & \sout{4.38} & \textbf{0.159} & \textbf{0.0238} & \textbf{4.24} \\
naval & 0.0215 & 0.857 & -1.84 & \sout{0.313} & \sout{0.215} & \sout{5.32} & \textbf{0.0215} & \textbf{0.188} & \textbf{6.68} \\
power plant & 0.215 & 0.0275 & -0.943 & \sout{0.223} & \sout{0.000265} & \sout{0.158} & \textbf{0.215} & \textbf{0.000225} & \textbf{0.155} \\
protein & 0.00193 & 0.4 & -0.919 & \sout{0.0484} & \sout{0.228} & \sout{3.92} & \textbf{0.00193} & \textbf{0.28} & \textbf{4.61} \\
superconductivity & 0.293 & 0.0162 & -0.97 & \sout{0.375} & \sout{0.000344} & \sout{0.344} & \textbf{0.293} & \textbf{0.00104} & \textbf{0.185} \\
wine-red & 0.762 & 0.00564 & -1.21 & \textbf{0.765} & 0.00197 & \textbf{-1.11} & \textbf{0.762} & \textbf{0.00181} & -1.14 \\
wine-white & 0.748 & 0.00216 & -1.21 & \sout{0.756} & \sout{0.00044} & \sout{-1.3} & \textbf{0.748} & \textbf{0.000356} & \textbf{-1.1} \\
yacht & 0.0244 & 0.383 & -0.919 & \sout{0.663} & \sout{0.096} & \sout{-0.744} & \textbf{0.0244} & \textbf{0.0376} & \textbf{2.18} \\
\textit{{Total wins or ties}} & -- & -- & -- & \textit{2} & \textit{0} & \textit{2} & \textit{11} & \textit{11} & \textit{9} \\
\bottomrule
\end{tabular}
}
\end{table}

\subsection{Adaptation to Monte Carlo Dropout Methods}\label{sec:adaptation-mcd}

\citet{gal2016dropout} propose using dropout during training and prediction to approximate a Gaussian Process.
During prediction, they generate $M$ mean estimates from the same network but for different dropout samples;
this results in a predictive uniform mixture of a mean-only model evaluated over $M$ dropout samples $d_m$:
\begin{align*}
    p(Y|X=x) = \frac{1}{M}\sum_{m=1}^M \N(Y | \mu(x,d_m,\theta), I_{\dim(Y)}).
\end{align*}
\citet{kendall2017uncertainties} replace the underlying homoscedastic model with a heteroscedastic model, which results in a predictive uniform mixture of a heteroscedastic model evaluated over $M$ dropout samples:
\begin{align}\label{eq:heteroscedastic-mcd}
    p(Y|X=x) = \frac{1}{M}\sum_{m=1}^M \N(Y | \mu(x,d_m,\theta), \Sigma(x,d_m,\theta)).
\end{align}
Both \citet{gal2016dropout} and \citet{kendall2017uncertainties} use conventional NLL minimization with the addition of dropout.
Because \citet{gal2016dropout} assume homoscedastic noise variance, they avoid the deleterious `rich-get-richer' cycle of using high variance to explain and ignore poor mean estimates~\citep{seitzer2022on}.
\citet{kendall2017uncertainties}, however, map covariates to heteroscedastic noise variance estimates and thus are susceptible to entering this cycle.
To construct a `faithful' variant, we simply optimize the underlying Gaussian model with our method and dropout.
Evaluation of our `faithful' variant's predictive distribution is identically over dropout samples as in supplement \cref{eq:heteroscedastic-mcd}.

\Cref{fig:covergence-mcd} examines the convergence behavior of the Monte Carlo Dropout methods.
Unlike for the other model classes, we use two hidden layers in $f_\text{trunk}$ for Monte Carlo Dropout methods during our convergence experiments.
Using dropout with just a single layer deeply affected performance of all Monte Carlo Dropout methods.
It is unsurprising that adding dropout does not resolve NLL optimization flaws.
Indeed, the Conventional Heteroscedastic Monte Carlo Dropout~\citep{kendall2017uncertainties} uses high variance to explain the mean error at $X=9.5$ despite having the flexibility to converge thereto as demonstrated by
the Unit Variance Homoscedastic Monte Carlo Dropout~\citep{gal2016dropout}.
As with the Unit Variance Homoscedastic Deep Ensemble, the Unit Variance Homoscedastic Monte Carlo Dropout model's predictive variance changes w.r.t. $x$ outside the data-rich interval $[2.5,7.5]$ to capture epistemic uncertainty in the mean.
\Cref{tab:uci-mcd} reports UCI regression performances for the Monte Carlo Dropout methods.
Our faithful variant overwhelmingly outperforms~\citet{kendall2017uncertainties}.

\begin{figure*}[h!]
    \centering
    \includegraphics[width=0.75\textwidth]{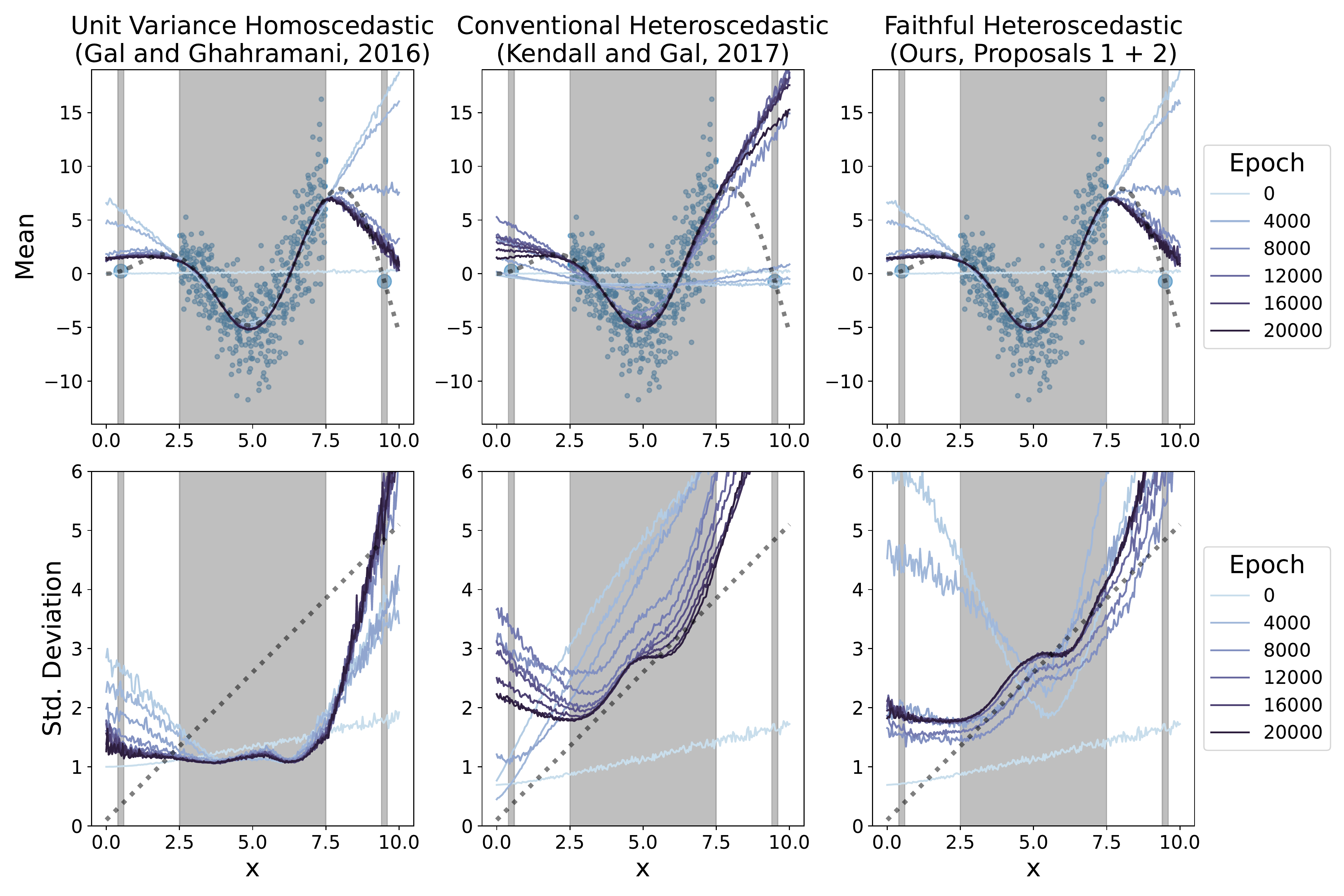}
    \caption{Convergence Behavior of Monte Carlo Dropout Methods}
    \label{fig:covergence-mcd}
\end{figure*}

\begin{table*}[h!]
    \centering
    \caption{UCI Regression Performance of Monte Carlo Dropout Methods}
    \label{tab:uci-mcd}
    \adjustbox{width=0.75\textwidth}{\begin{tabular}{l|ccc|ccc|ccc}
\toprule
 & \multicolumn{3}{|c}{Unit Variance Homoscedastic} & \multicolumn{3}{|c}{Conventional Heteroscedastic} & \multicolumn{3}{|c}{Faithful Heteroscedastic} \\
Dataset & RMSE & ECE & LL & RMSE & ECE & LL & RMSE & ECE & LL \\
\midrule
boston & 0.32 & 0.172 & -0.975 & \sout{0.374} & \sout{0.00793} & \sout{-0.446} & \textbf{0.32} & \textbf{0.0053} & \textbf{-5.86} \\
carbon & 0.0487 & 0.48 & -2.76 & \textbf{0.0488} & 0.0169 & \textbf{9.35} & \textbf{0.0487} & \textbf{0.0111} & 4.99 \\
concrete & 0.25 & 0.154 & -0.955 & \sout{0.292} & \sout{0.00165} & \sout{-0.725} & \textbf{0.25} & \textbf{0.00374} & \textbf{-0.267} \\
energy & 0.159 & 0.327 & -1.85 & \sout{0.192} & \sout{0.0331} & \sout{4.38} & \textbf{0.159} & \textbf{0.0238} & \textbf{4.24} \\
naval & 0.0215 & 0.857 & -1.84 & \sout{0.313} & \sout{0.215} & \sout{5.32} & \textbf{0.0215} & \textbf{0.188} & \textbf{6.68} \\
power plant & 0.215 & 0.0275 & -0.943 & \sout{0.223} & \sout{0.000265} & \sout{0.158} & \textbf{0.215} & \textbf{0.000225} & \textbf{0.155} \\
protein & 0.00193 & 0.4 & -0.919 & \sout{0.0484} & \sout{0.228} & \sout{3.92} & \textbf{0.00193} & \textbf{0.28} & \textbf{4.61} \\
superconductivity & 0.293 & 0.0162 & -0.97 & \sout{0.375} & \sout{0.000344} & \sout{0.344} & \textbf{0.293} & \textbf{0.00104} & \textbf{0.185} \\
wine-red & 0.762 & 0.00564 & -1.21 & \textbf{0.765} & 0.00197 & \textbf{-1.11} & \textbf{0.762} & \textbf{0.00181} & -1.14 \\
wine-white & 0.748 & 0.00216 & -1.21 & \sout{0.756} & \sout{0.00044} & \sout{-1.3} & \textbf{0.748} & \textbf{0.000356} & \textbf{-1.1} \\
yacht & 0.0244 & 0.383 & -0.919 & \sout{0.663} & \sout{0.096} & \sout{-0.744} & \textbf{0.0244} & \textbf{0.0376} & \textbf{2.18} \\
\textit{{Total wins or ties}} & -- & -- & -- & \textit{2} & \textit{0} & \textit{2} & \textit{11} & \textit{11} & \textit{9} \\
\bottomrule
\end{tabular}
}
\end{table*}

\subsection{Adaptation to Student Methods}\label{sec:adaptation-sudent}

In section 2 of our main manuscript, we discuss existing methods for improving heteroscedastic regression performance that yield a predictive Student distribution~\citep{takahashi2018student,skafte2019reliable,stirn2020variational}.
Our proposals are also adaptable to the Student distribution.
To do so, we add a new parameter head such that the network partitions are now:
\begin{itemize}
    \item Shared representation learner $f_{\text{trunk}}$ has parameters $\theta_z$.
    \item Location parameter head $f_\mu$ has parameters $\theta_\mu$.
    \item Scale parameter head $f_\sigma$ has parameters $\theta_\sigma$.
    \item Degrees-of-freedom parameter head $f_\nu$ has parameters $\theta_\nu$, for which we constrain outputs to $(3, \infty)$ via a shifted softplus (we do this for our implementations of~\citet{takahashi2018student,stirn2020variational} as well).
\end{itemize}
Similar to our proposal for heteroscedastic Normal models, we can guarantee faithfulness when optimizing a heteroscedastic Student model's log likelihood $\LL$ by:
\begin{enumerate}
    \item Setting $\nabla_{\mu(x)} \LL \coloneqq \sum\limits_{(x,y)\in\D} \mu(x) - y$ (i.e.~the gradient resulting from minimizing the sum of squared errors)
    \item Stopping $\nabla_{\sigma(x)} \LL$ and $\nabla_{\nu(x)} \LL$ from contributing to updates for any shared parameters $\theta_z$
\end{enumerate}
The first proposal is equivalent to optimizing a Student distribution in the limit as the degrees of freedom $\rightarrow \infty$.
At prediction time, our Unit Variance Homoscedastic Student variant approximates this limit by outputting a constant 100 for the degrees-of-freedom parameter and setting $\sigma = \sqrt{(100 - 2) / 100}$ to ensure unit variance.

\Cref{fig:covergence-student} examines the convergence behavior of the Student methods.
Our Faithful Heteroscedastic Student is the only heteroscedastic method that can converge on the isolated point at $X=9.5$.
\Cref{tab:uci-student} reports UCI regression performances for the Student methods.
Our Faithful Heteroscedastic Student is never empirically unfaithful and is the top-performing method.

\begin{figure*}[h!]
    \centering
    \includegraphics[width=\textwidth]{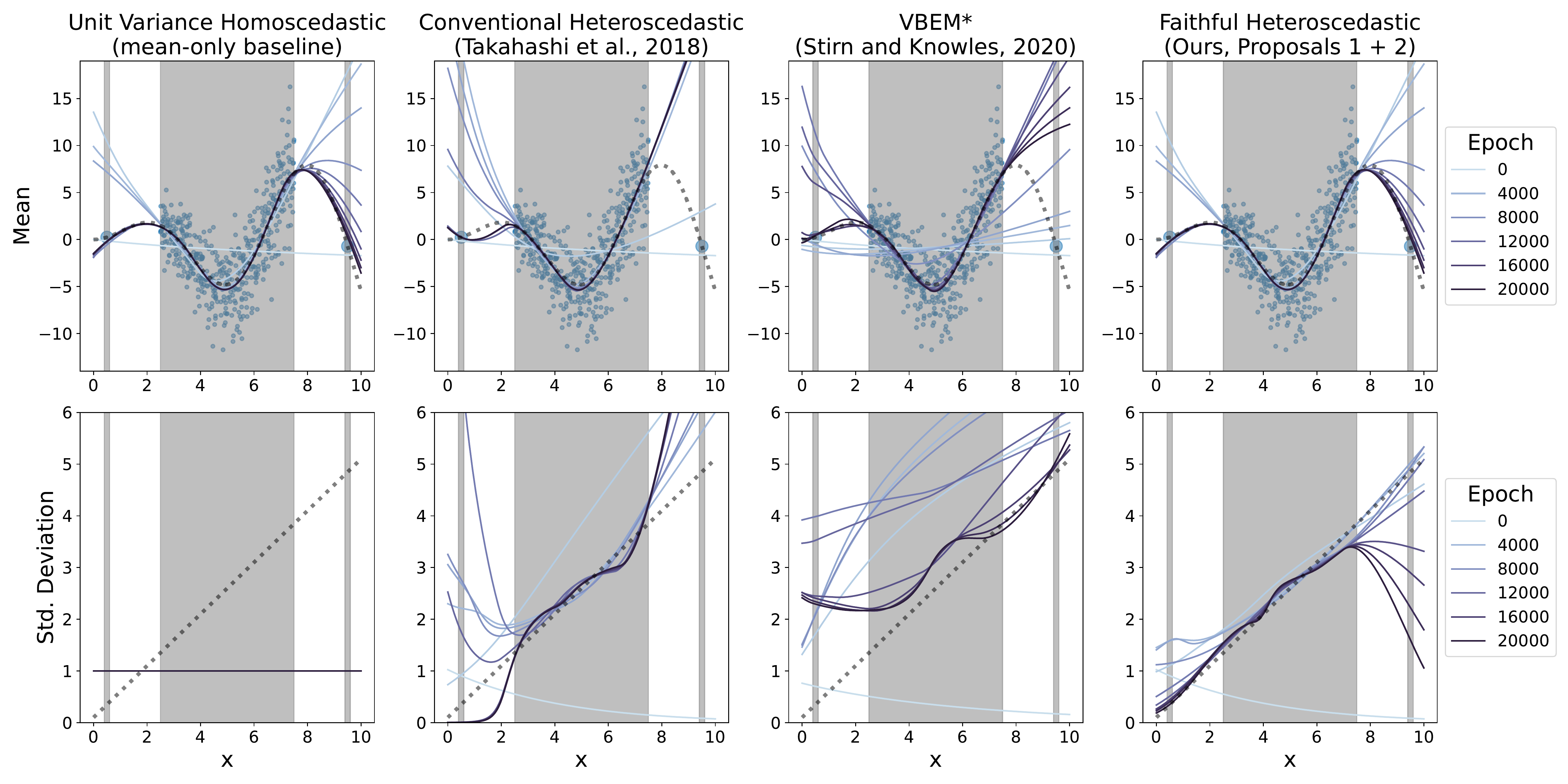}
    \caption{Convergence Behavior of Student Methods}
    \label{fig:covergence-student}
\end{figure*}

\begin{table*}[h!]
    \centering
    \caption{UCI Regression Performance of Student Methods}
    \label{tab:uci-student}
    \adjustbox{width=\textwidth}{\begin{tabular}{l|ccc|ccc|ccc|ccc}
\toprule
 & \multicolumn{3}{|c}{Unit Variance Homoscedastic} & \multicolumn{3}{|c}{Conventional Heteroscedastic} & \multicolumn{3}{|c}{VBEM*} & \multicolumn{3}{|c}{Faithful Heteroscedastic} \\
Dataset  & RMSE & ECE & LL & RMSE & ECE & LL & RMSE & ECE & LL & RMSE & ECE & LL \\
\midrule
boston & 0.304 & 0.168 & -0.959 & \sout{0.387} & \sout{0.017} & \sout{-1.05} & \textbf{0.339} & \textbf{0.0137} & \textbf{-0.311} & \textbf{0.304} & 0.026 & -1.63 \\
carbon & 0.0489 & 0.279 & -2.74 & \textbf{0.0488} & 0.000339 & \textbf{12.6} & \textbf{0.0488} & 0.0054 & \textbf{9.01} & 0.0489 & \textbf{0.000134} & \textbf{11.8} \\
concrete & 0.265 & 0.127 & -0.947 & \sout{0.322} & \sout{0.0151} & \sout{-2.53} & \textbf{0.27} & \textbf{0.0134} & \textbf{-0.22} & \textbf{0.265} & 0.0306 & \textbf{-0.93} \\
energy & 0.16 & 0.268 & -1.84 & \sout{0.196} & \sout{0.00164} & \sout{4.48} & \sout{0.185} & \sout{0.0228} & \sout{2.96} & \textbf{0.16} & \textbf{0.00104} & \textbf{3.94} \\
naval & 0.0247 & 0.271 & -1.82 & \sout{0.521} & \sout{0.000101} & \sout{4.69} & \sout{0.0482} & \sout{0.0034} & \sout{6.34} & \textbf{0.0247} & \textbf{0.000171} & \textbf{7.04} \\
power plant & 0.22 & 0.0105 & -0.936 & \sout{0.227} & \sout{0.000179} & \sout{0.15} & \sout{0.221} & \sout{0.000177} & \sout{0.122} & \textbf{0.22} & \textbf{0.000152} & \textbf{0.154} \\
protein & 0.00257 & 0.49 & -0.911 & \sout{0.0287} & \sout{0.000721} & \sout{4} & \sout{0.00383} & \sout{0.00429} & \sout{4.38} & \textbf{0.00257} & \textbf{0.00137} & \textbf{4.94} \\
superconductivity & 0.326 & 0.0104 & -0.966 & \sout{0.397} & \sout{0.00153} & \sout{0.16} & \sout{0.347} & \sout{0.0011} & \sout{0.132} & \textbf{0.326} & \textbf{0.000993} & \textbf{0.02} \\
wine-red & 0.769 & 0.00601 & -1.21 & \sout{0.782} & \sout{0.00249} & \sout{-1.36} & \textbf{0.771} & 0.00567 & \textbf{-1.16} & \textbf{0.769} & \textbf{0.00247} & -1.19 \\
wine-white & 0.777 & 0.0015 & -1.22 & \sout{0.783} & \sout{0.000786} & \sout{-1.29} & \textbf{0.777} & 0.000711 & \textbf{-1.16} & \textbf{0.777} & \textbf{0.00041} & -1.17 \\
yacht & 0.0226 & 0.375 & -0.912 & \sout{0.677} & \sout{0.0603} & \sout{-0.689} & \textbf{0.0193} & 0.103 & \textbf{2.42} & 0.0226 & \textbf{0.0161} & 2.29 \\
\textit{{Total wins or ties}} & -- & -- & -- & \textit{1} & \textit{0} & \textit{1} & \textit{6} & \textit{2} & \textit{6} & \textit{9} & \textit{9} & \textit{7} \\
\bottomrule
\end{tabular}
}
\end{table*}

\section{EXPERIMENT DETAILS}\label{sec:experiment-details}
In the following subsections, we provide additional details needed to reproduce our experiments.
To most exactly reproduce our results, please see or use our provided code.
By default, our code uses the same random number seeds that we used and enables GPU determinism.
In all experiments, we assume diagonal covariance and parameterize its standard deviation, that is we replace (co)variance function $f_\Sigma$ with diagonal standard deviation function $f_\sigma$.

\subsection{Convergence Experiment Details}\label{subsec:convergence-details}
Supplement \cref{tab:arch-convergence} has the neural network architecture used to generate figure 2 in our main manuscript.
Every method fits the same 500 data points, which fully comprise a batch, and submits their objective to Adam~\citep{kingma2014adam} with a 1e-3 learning rate for 20,000 epochs.
Our main manuscript contains the remaining implementation details.
\begin{table}[ht!]
    \centering
    \caption{Neural Network Architecture: Convergence Experiment}
    \begin{tabular}{lll}
    \toprule
    \multicolumn{1}{c}{$f_\text{trunk}$} & \multicolumn{1}{c}{$f_\mu$} & \multicolumn{1}{c}{$f_\sigma$} \\
    \midrule
    Dense (50 elu units) & Dense (1 linear unit) & Dense (1 softplus unit)\\
    \bottomrule
    \end{tabular}\label{tab:arch-convergence}
\end{table}

\subsection{UCI Experiment Details}\label{subsec:uci-details}
Supplement \cref{tab:arch-uci} has the neural network architecture used in section 3.2 of our main manuscript.
We use Adam~\citep{kingma2014adam} with a 1e-3 learning rate for a maximum of 60k epochs.
Because the UCI datasets are relatively small, every batch contains the entire training set.
We stop training early if the validation root mean square error (RMSE) has not improved for 100 epochs.
We always restore the model weights from the epoch with the best RMSE.
We monitor RMSE instead of NLL in order to give every model the best chance of being faithful.

\begin{table}[ht!]
    \centering
    \caption{Neural Network Architecture: UCI Experiments}
    \begin{tabular}{lll}
    \toprule
    \multicolumn{1}{c}{$f_\text{trunk}$} & \multicolumn{1}{c}{$f_\mu$} & \multicolumn{1}{c}{$f_\sigma$} \\
    \midrule
    Dense (50 elu units) & Dense ($\dim(Y)$ linear units) & Dense ($\dim(Y)$ softplus units)\\
    Dense (50 elu units) \\
    \bottomrule
    \end{tabular}\label{tab:arch-uci}
\end{table}

\subsection{VAE Experiment Details}\label{subsec:vae-details}
Supplement \cref{tab:arch-vae} has the neural network architecture used in section 3.3 of our main manuscript.
The encoder, $f_\text{trunk}$, parameterizes a $\dim(z)=16$ latent embedding.
We use Adam~\citep{kingma2014adam} with a 1e-3 learning rate for a maximum of 2k epochs with a 2048 batch size.
We monitor RMSE, perform early stopping, and weight restoration as discussed in supplement \cref{subsec:uci-details}.
We rotate our noise template by $\frac{2\pi}{\text{num. classes.}}(\text{class label})$ radians.
The training/validation split assignments are those provided by the TensorFlow Datasets API.

\begin{table}[ht!]
    \centering
    \caption{Neural Network Architecture: VAE Experiments}
    \begin{tabular}{lll}
    \toprule
    \multicolumn{1}{c}{$f_\text{trunk}$} & \multicolumn{1}{c}{$f_\mu$} & \multicolumn{1}{c}{$f_\sigma$} \\
    \midrule
    Conv ($32\times5\times5$, stride 2, elu) & Dense (128 elu units) & Dense (128 elu units)\\
    Conv ($32\times5\times5$, stride 2, elu) & Dense (1568 elu units) & Dense (1568 elu units)\\
    Dense (128 elu units) & ConvT ($32\times5\times5$, stride 2, elu) & ConvT ($32\times5\times5$, stride 2, elu)\\
    Dense (16 linear units, 16 softplus units) & ConvT ($1\times5\times5$, stride 2, linear) & ConvT ($1\times5\times5$, stride 2, softplus)\\
    Sampling Layer\\
    \bottomrule
    \end{tabular}\label{tab:arch-vae}
\end{table}

\subsection{CRISPR-Cas13 Experiment Details}\label{subsec:crispr-details}
Supplement \cref{tab:arch-crispr} has the neural network architecture used in section 3.4 of our main manuscript.
We use Adam~\citep{kingma2014adam} with a 1e-3 learning rate for a maximum of 2k epochs with a 2048 batch size.
We monitor RMSE, perform early stopping, and weight restoration as discussed in supplement \cref{subsec:uci-details}.
For each dataset, we have either two or three replicate measurements for computing the mean (i.e.~de-noised) efficacy scores.
We use 10k randomly selected training points as the background to compute SHAP values~\citep{lundberg_unified_2017} for a validation fold.
If we have fewer than 10k training points, we use the entire training set.
Their \href{https://github.com/slundberg/shap#deep-learning-example-with-deepexplainer-tensorflowkeras-models}{online documentation} recommends using training points as the background.
For visual aesthetics, we add back a scaled version of the average model output over the background samples to the respective sequence SHAP values.
We scale this value by $\frac{1}{\text{sequence length}}$.

\begin{table}[ht!]
    \centering
    \caption{Neural Network Architecture: CRISPR-Cas13 Experiments}
    \begin{tabular}{lll}
    \toprule
    \multicolumn{1}{c}{$f_\text{trunk}$} & \multicolumn{1}{c}{$f_\mu$} & \multicolumn{1}{c}{$f_\sigma$} \\
    \midrule
    Conv ($64\times4$, stride 1, relu) & Dense (128 elu units) & Dense (128 elu units)\\
    Conv($64\times4$, stride 1, relu) & Dense (32 elu units) & Dense (32 elu units)\\
    MaxPool1D (pool 2, stride 1) & Dense (1 linear unit) & Dense (1 softplus unit)\\
    \bottomrule
    \end{tabular}\label{tab:arch-crispr}
\end{table}

\section{ADDITIONAL RESULTS}\label{sec:additional-results}
The following pages provide results supplementary to those from our main manuscript.

\subsection{Additional VAE Results}\label{subsec:vae-additional}
Supplement \cref{fig:vae-moments-mnist,fig:vae-noise-mnist} are larger versions of figures 3 and 4 from our main manuscript.
Supplement \cref{fig:vae-moments-fashion-mnist,fig:vae-noise-fashion-mnist} are the same but for the Fashion MNIST dataset.

\subsection{Additional CRISPR-Cas13 Results}\label{subsec:crispr-additional}
Supplement \cref{fig:crispr-flow-cytometry,fig:crispr-junction-targets,fig:crispr-off-target} contain the complete set of figures (those similar to figure 5 of our main manuscript) for all models and all CRISPR-Cas13 datasets.
The first two rows of every subplot are the average SHAP values of the estimated mean when trained on replicate and mean efficacy scores, respectively.
The next two rows are the same but for the average SHAP values of the estimated standard deviation.
The final row subtracts the average SHAP values of the estimated standard deviation when trained on means from those when trained on replicates and represents the SHAP values of the estimated square root of noise variance.

\newpage

\begin{figure}[ht!]
    \centering
    \includegraphics[width=\textwidth]{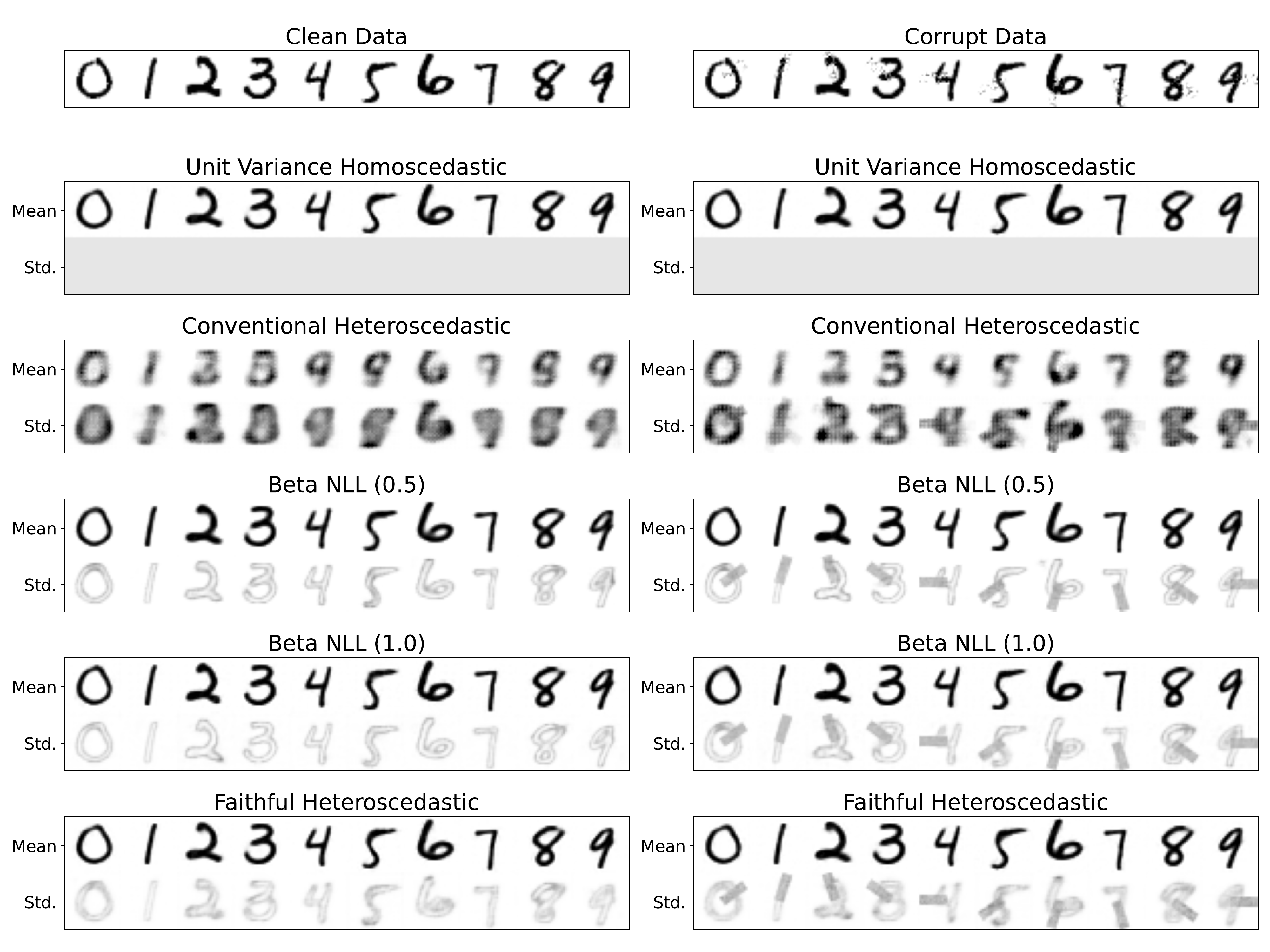}
    \caption{VAE Predictive Moments for MNIST}
    \label{fig:vae-moments-mnist}
\end{figure}

\begin{figure}[ht!]
    \centering
    \includegraphics[width=0.5\textwidth]{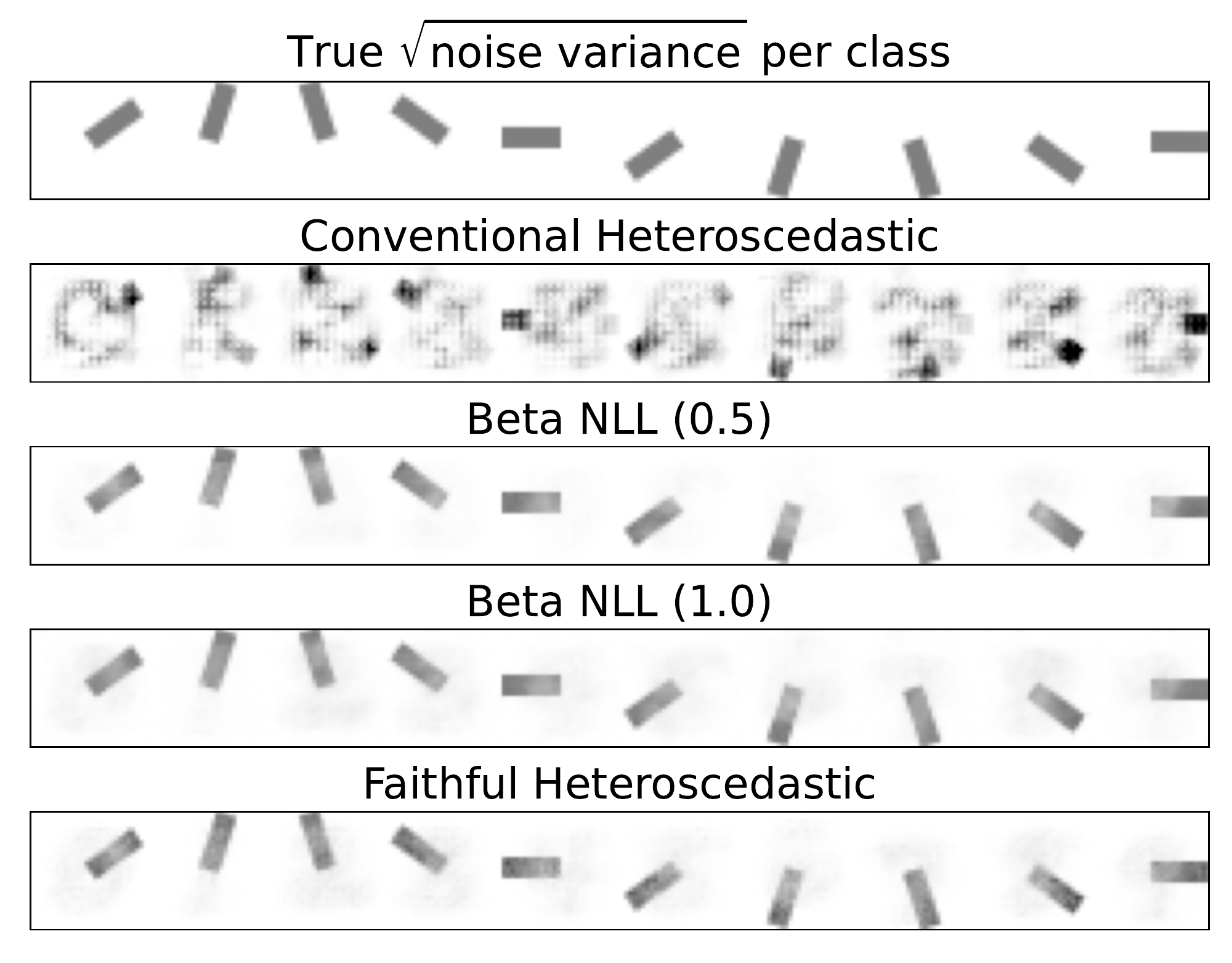}
    \caption{VAE Recovered Noise Variance for MNIST}
    \label{fig:vae-noise-mnist}
\end{figure}

\newpage

\begin{figure}[ht!]
    \centering
    \includegraphics[width=\textwidth]{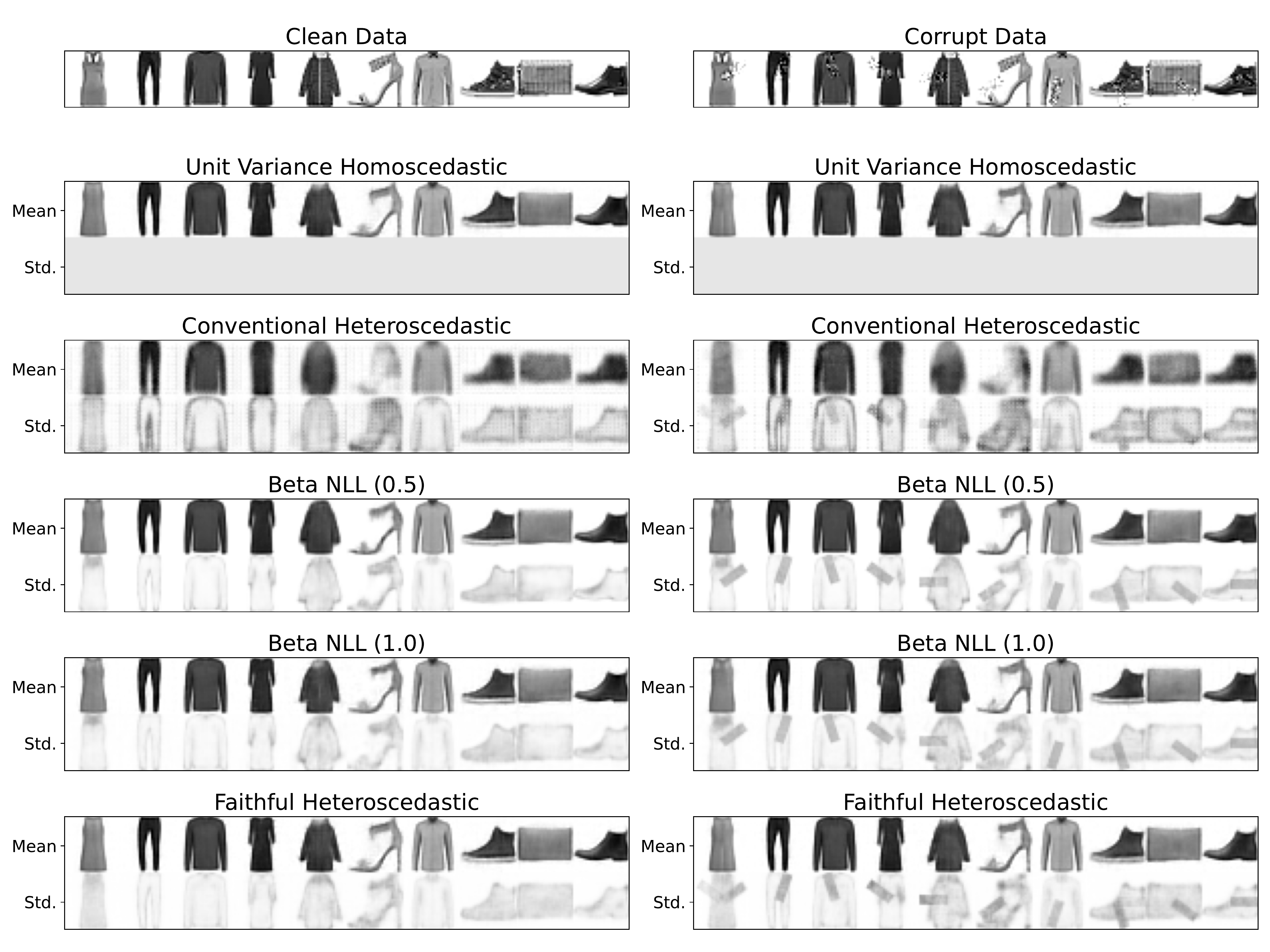}
    \caption{ VAE Predictive Moments for Fashion MNIST}
    \label{fig:vae-moments-fashion-mnist}
\end{figure}

\begin{figure}[ht!]
    \centering
    \includegraphics[width=0.5\textwidth]{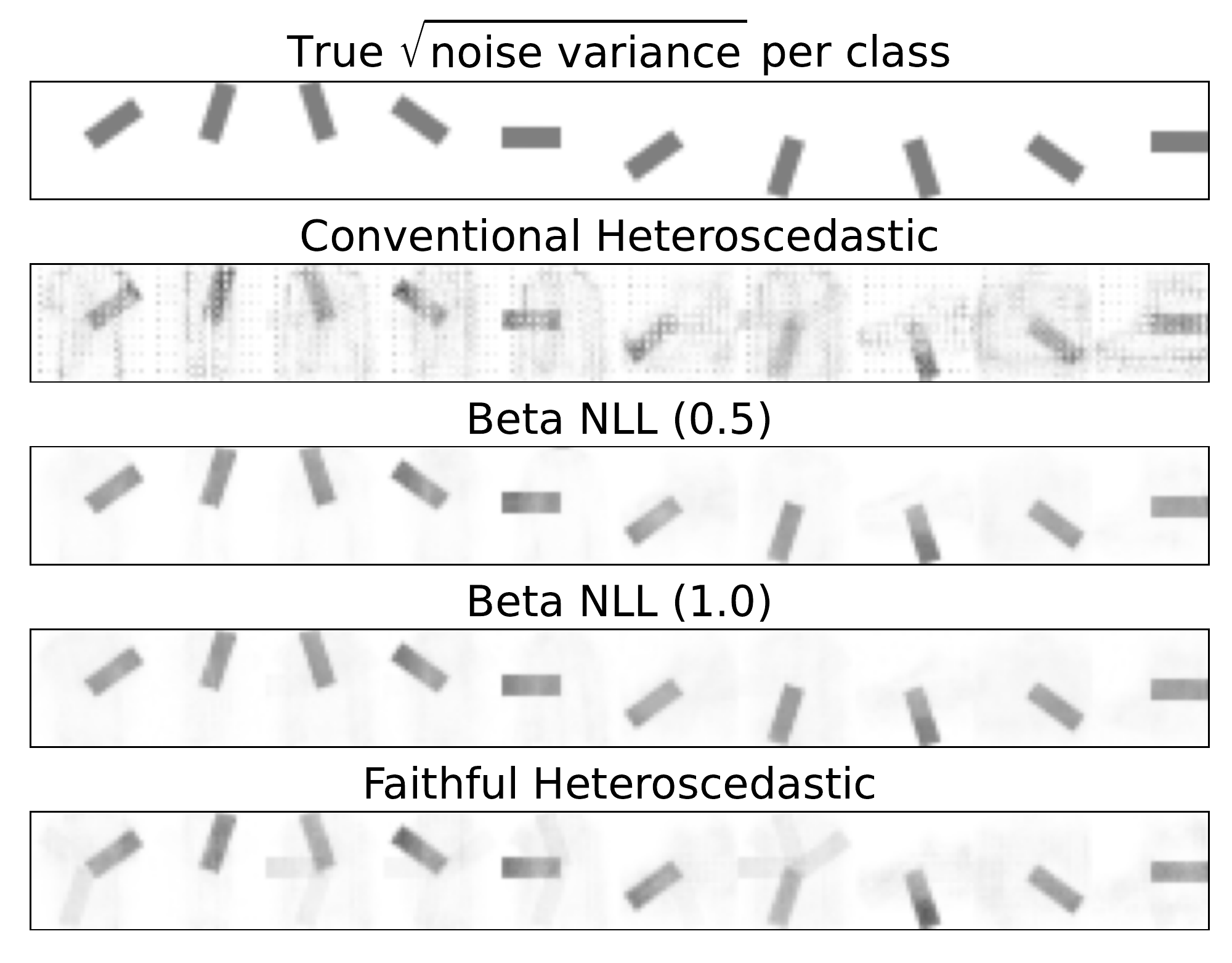}
    \caption{VAE Recovered Noise Variance for Fashion MNIST}
    \label{fig:vae-noise-fashion-mnist}
\end{figure}

\newpage

\begin{figure}[ht!]
    \centering
    \begin{subfigure}[b]{0.3\textwidth}
        \centering
        \includegraphics[width=\textwidth]{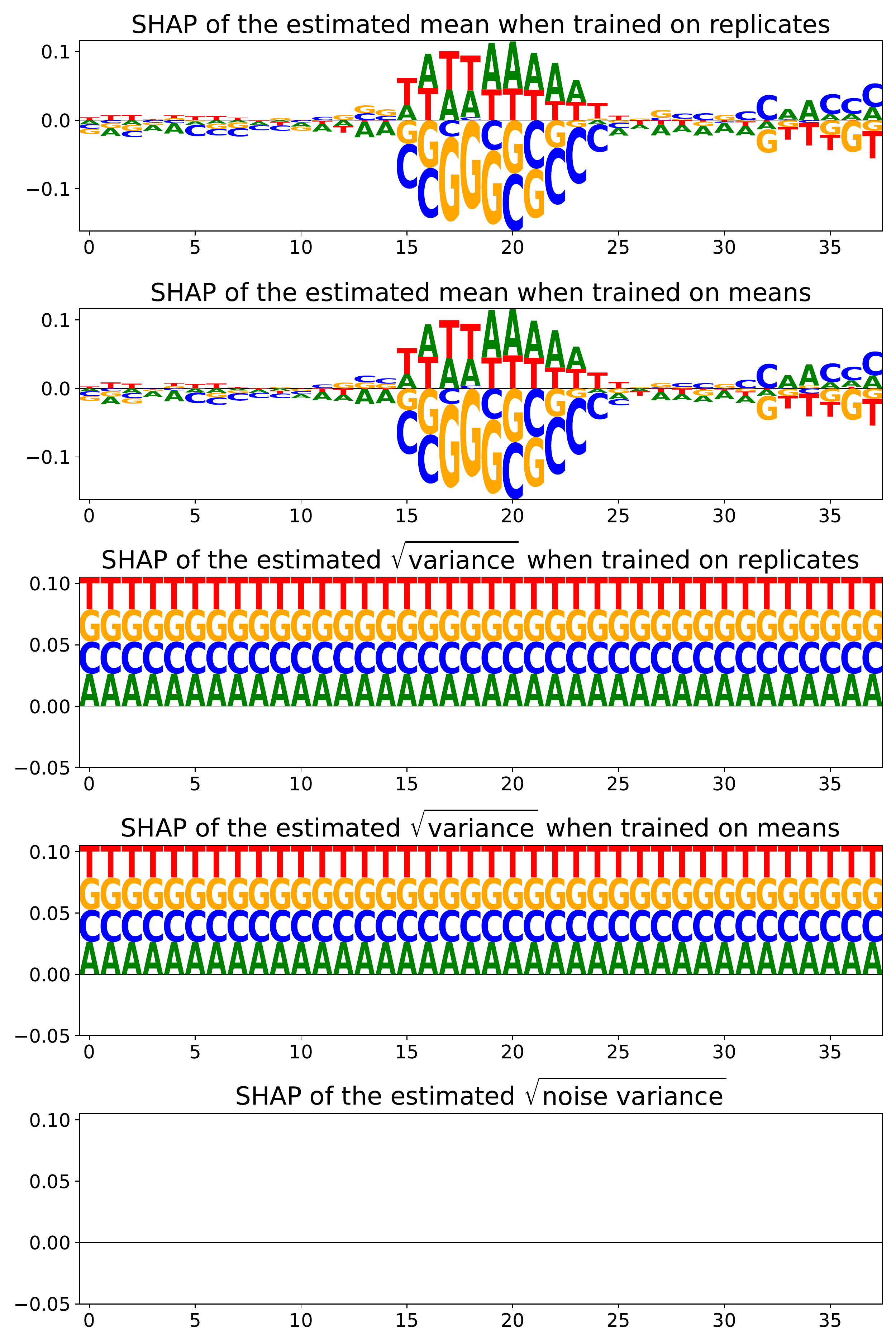}
        \caption{Unit Variance}
    \end{subfigure}
    \hfill
    \begin{subfigure}[b]{0.3\textwidth}
        \centering
        \includegraphics[width=\textwidth]{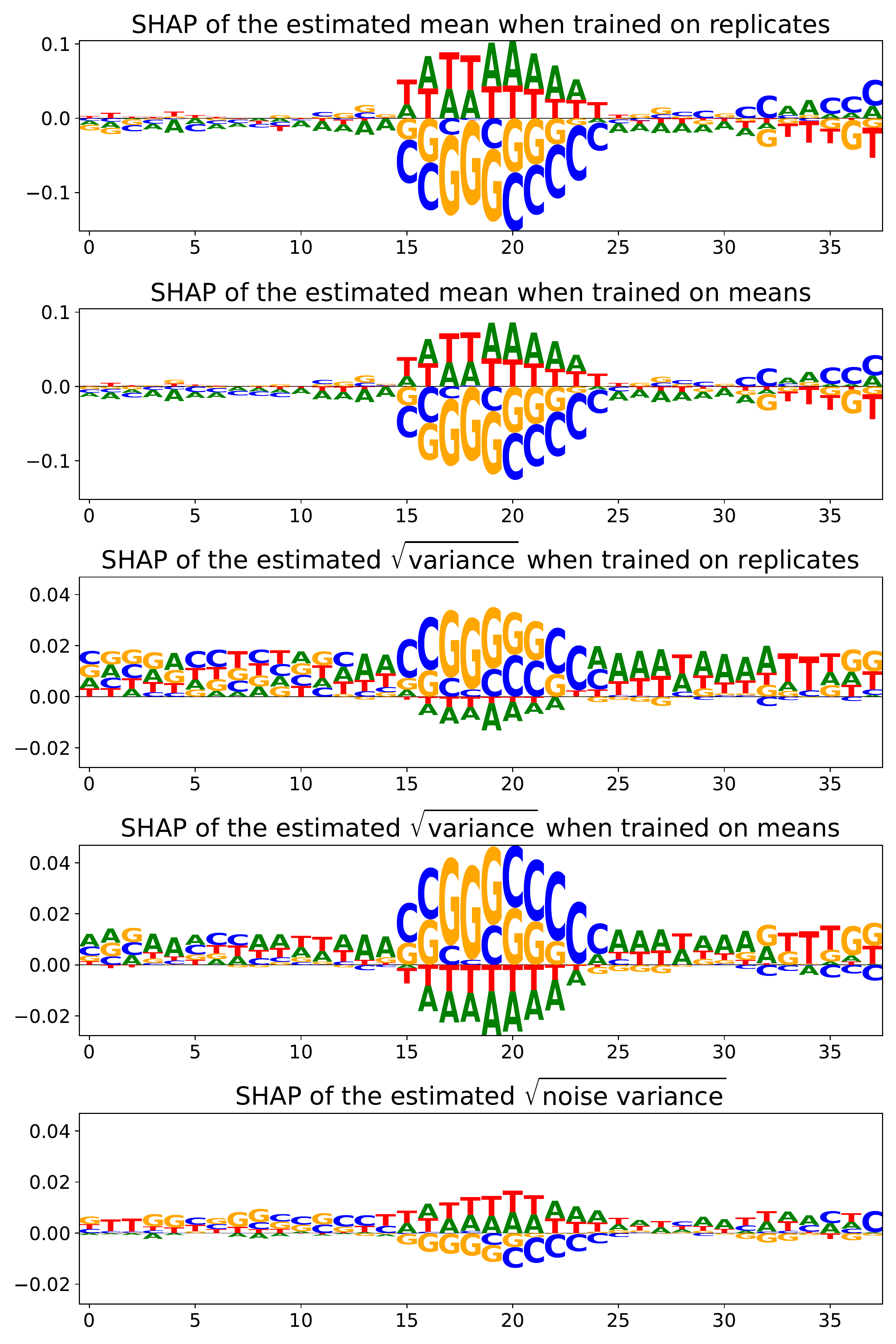}
        \caption{Heteroscedastic}
    \end{subfigure}
    \hfill
    \begin{subfigure}[b]{0.3\textwidth}
        \centering
        \includegraphics[width=\textwidth]{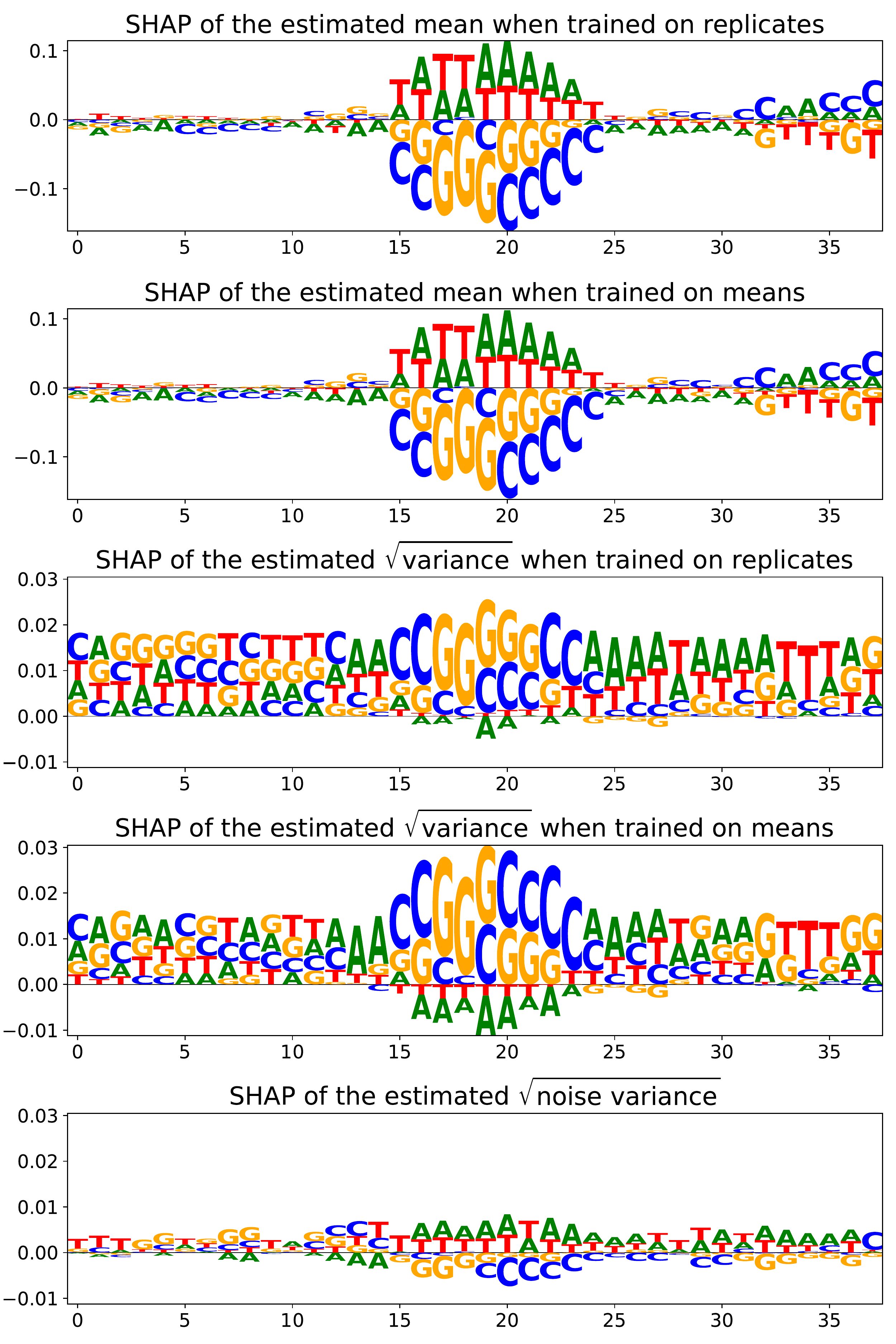}
        \caption{Beta NLL (0.5)}
    \end{subfigure}
    \\
     \begin{subfigure}[b]{0.3\textwidth}
        \centering
        \includegraphics[width=\textwidth]{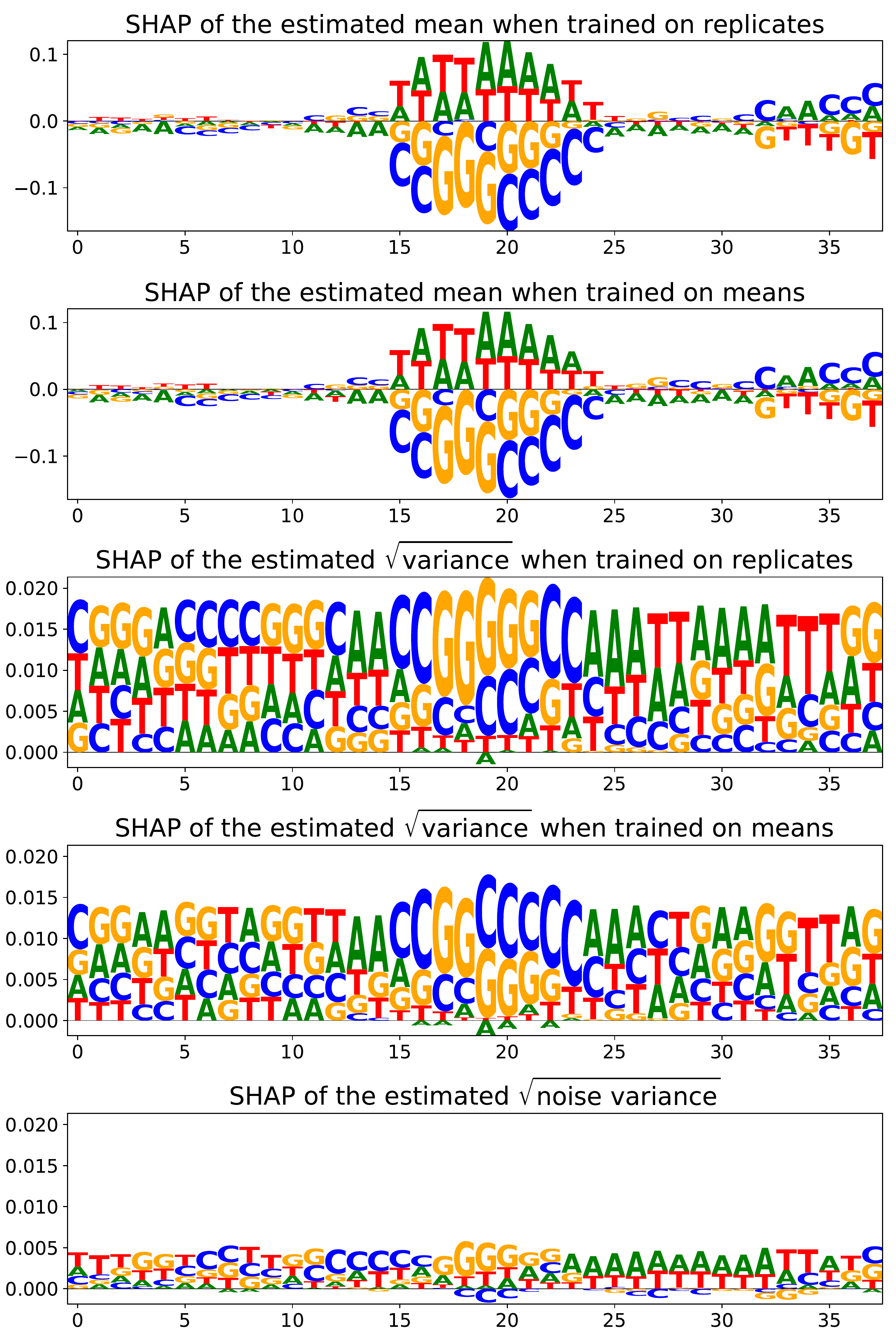}
        \caption{Beta NLL (1.0)}
    \end{subfigure}
    \vspace{20pt}
    \begin{subfigure}[b]{0.3\textwidth}
        \centering
        \includegraphics[width=\textwidth]{results/crispr_Normal_flow-cytometry_FaithfulHeteroscedastic}
        \caption{Faithful Heteroscedastic}
    \end{subfigure}
    \caption{SHAP Values for Flow Cytometry (HEK293) Dataset}
    \label{fig:crispr-flow-cytometry}
\end{figure}

\newpage

\begin{figure}[ht!]
    \centering
    \begin{subfigure}[b]{0.3\textwidth}
        \centering
        \includegraphics[width=\textwidth]{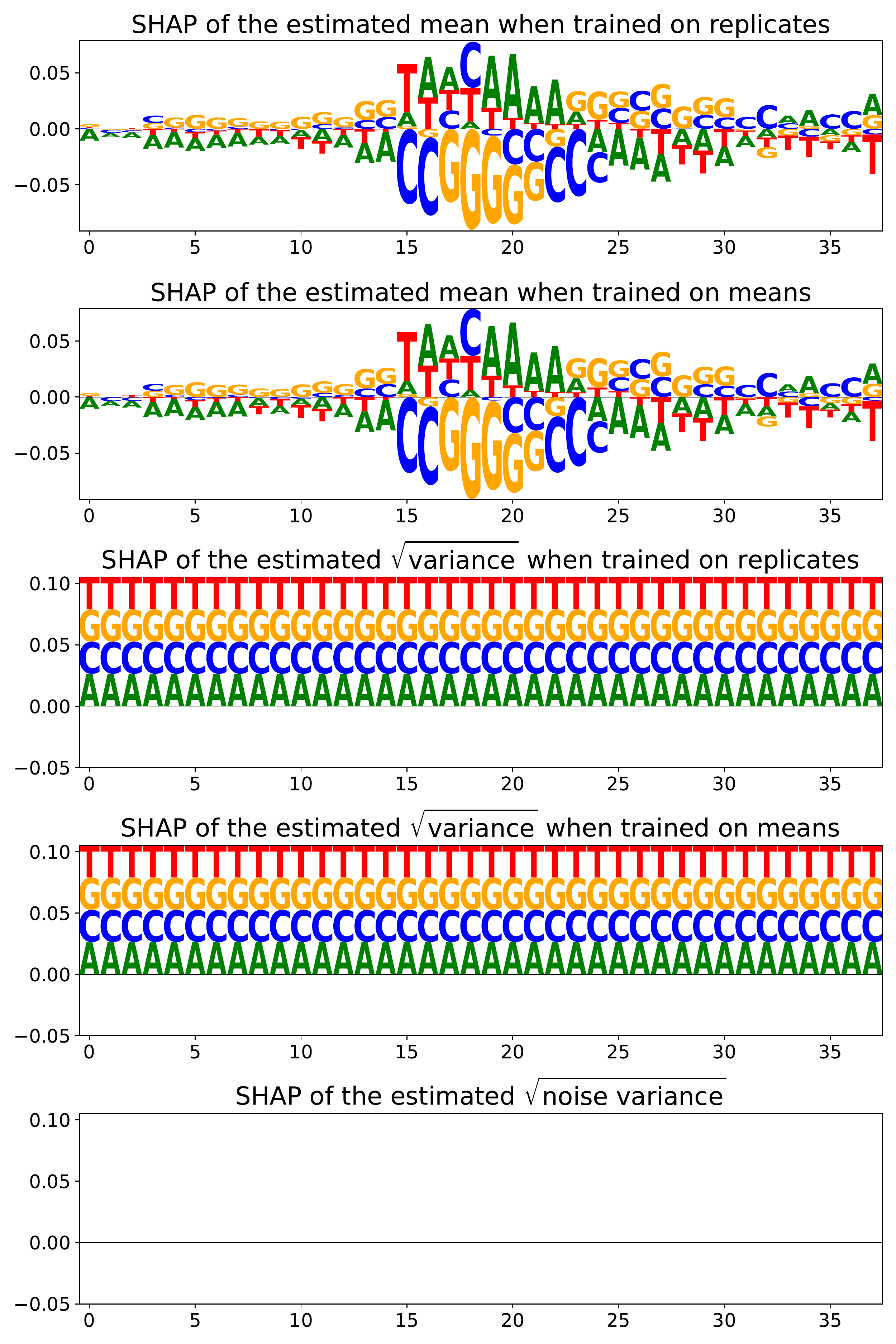}
        \caption{Unit Variance}
    \end{subfigure}
    \hfill
    \begin{subfigure}[b]{0.3\textwidth}
        \centering
        \includegraphics[width=\textwidth]{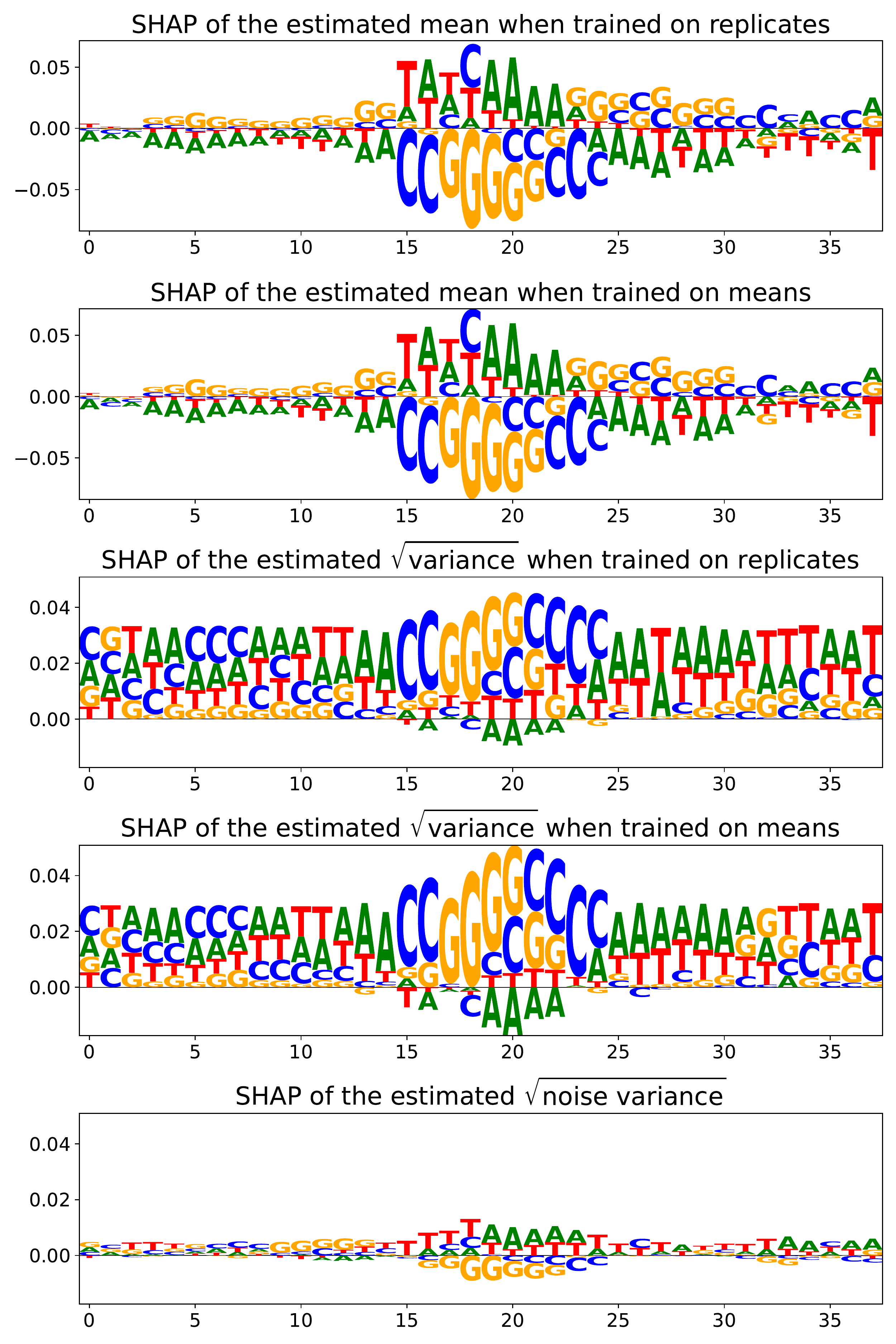}
        \caption{Heteroscedastic}
    \end{subfigure}
    \hfill
    \begin{subfigure}[b]{0.3\textwidth}
        \centering
        \includegraphics[width=\textwidth]{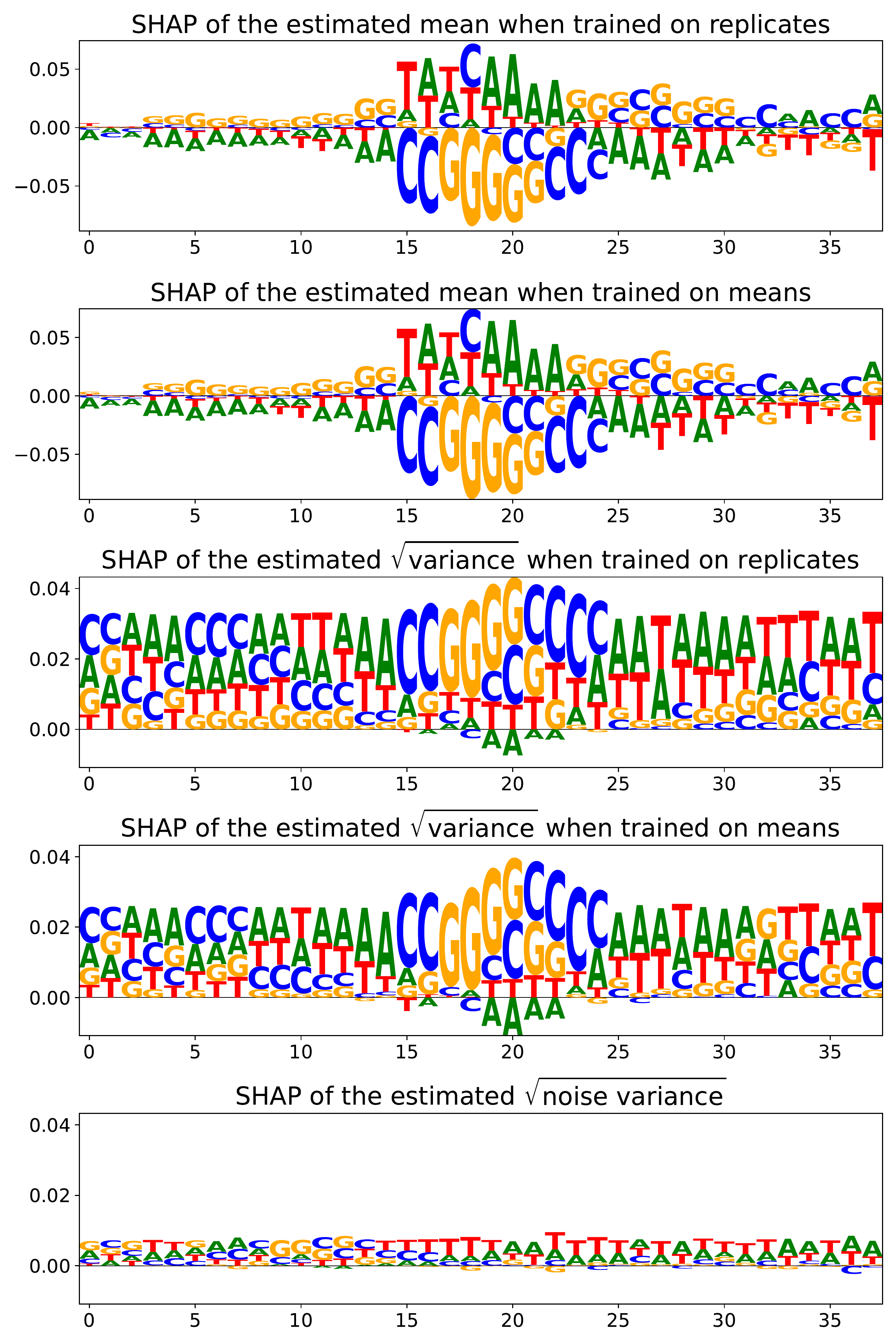}
        \caption{Beta NLL (0.5)}
    \end{subfigure}
    \\
     \begin{subfigure}[b]{0.3\textwidth}
        \centering
        \includegraphics[width=\textwidth]{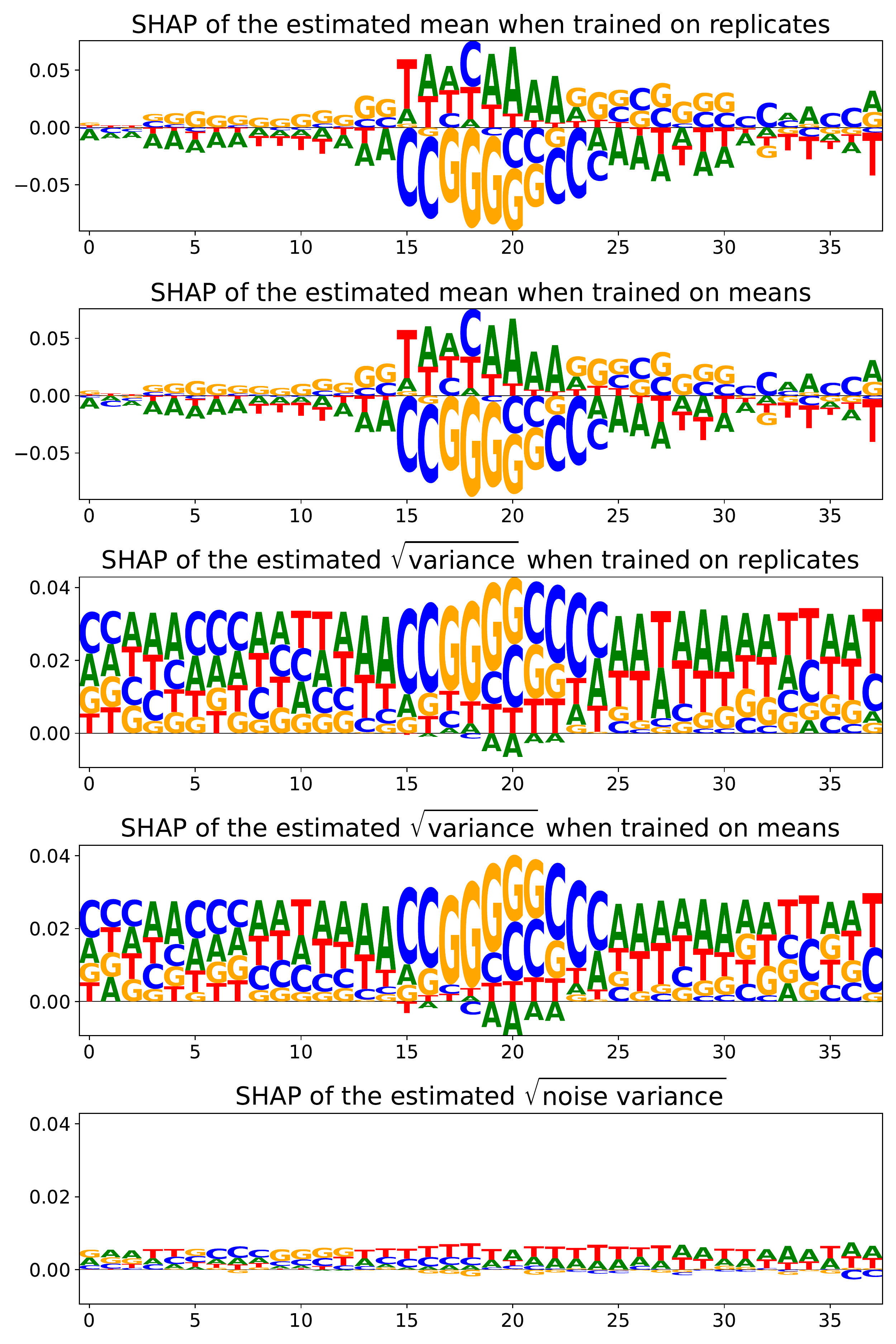}
        \caption{Beta NLL (1.0)}
    \end{subfigure}
    \vspace{20pt}
    \begin{subfigure}[b]{0.3\textwidth}
        \centering
        \includegraphics[width=\textwidth]{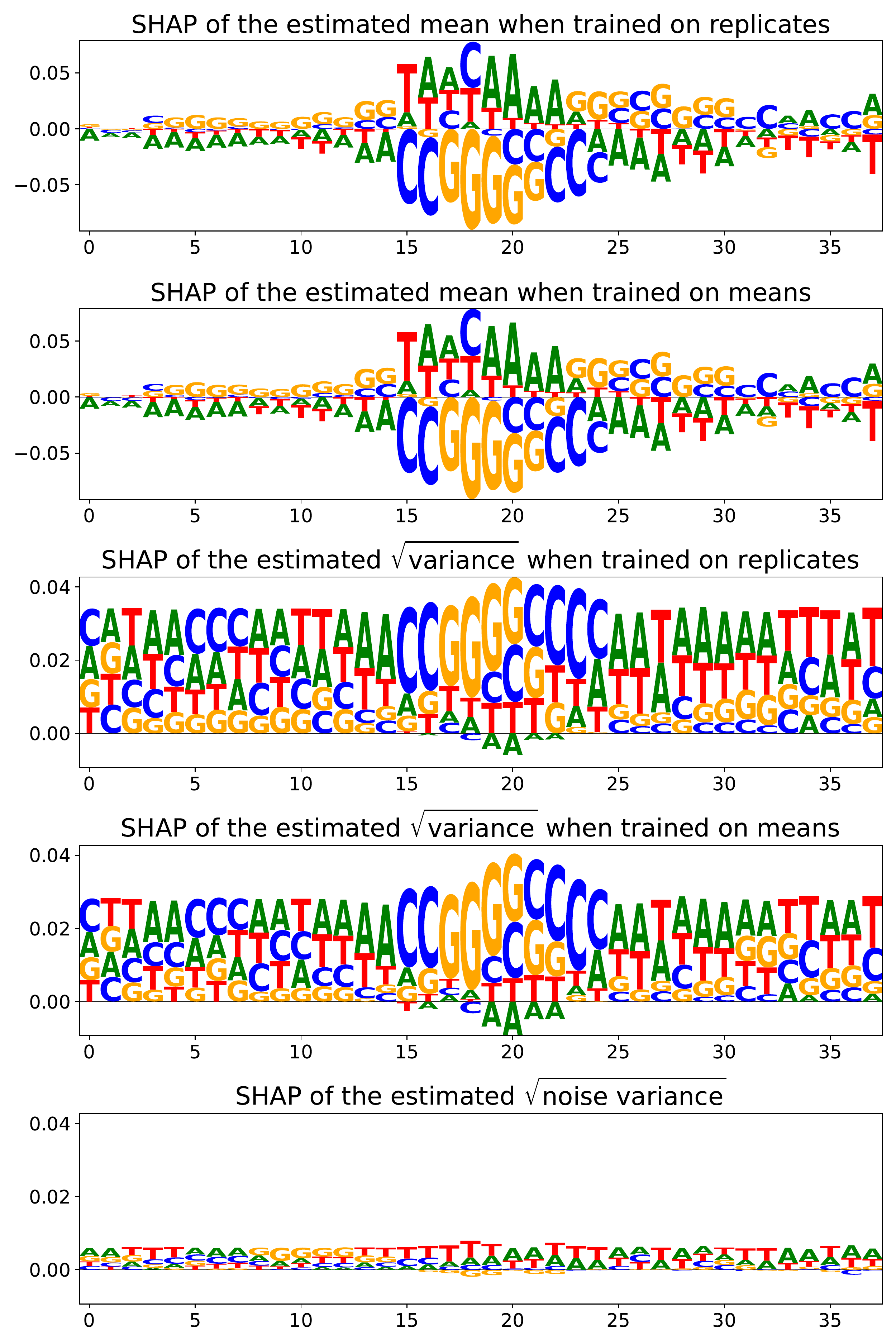}
        \caption{Faithful Heteroscedastic}
    \end{subfigure}
    \caption{SHAP Values for Survival Screen (A375) Dataset}
    \label{fig:crispr-junction-targets}
\end{figure}

\begin{figure}[ht!]
    \centering
    \begin{subfigure}[b]{0.3\textwidth}
        \centering
        \includegraphics[width=\textwidth]{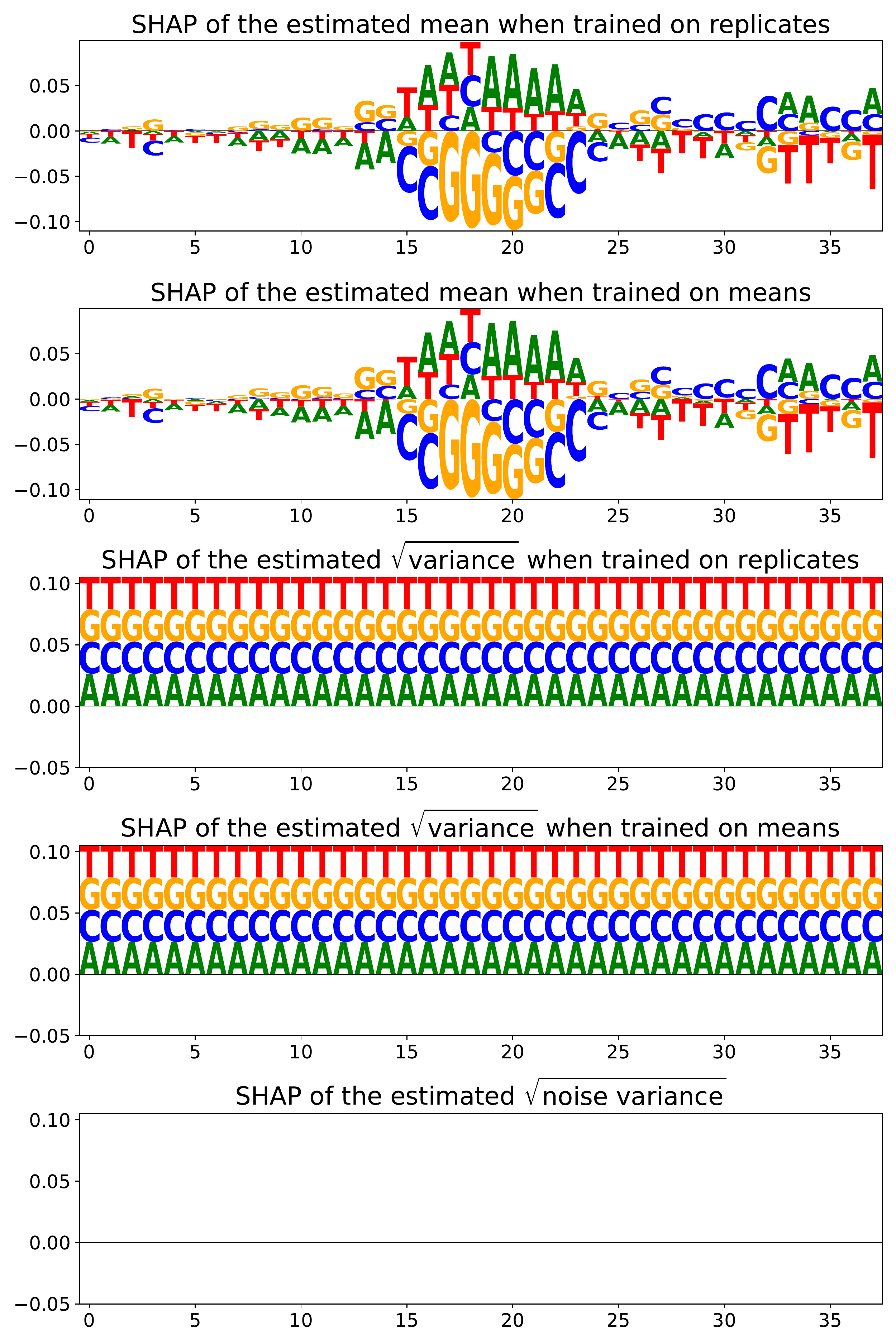}
        \caption{Unit Variance}
    \end{subfigure}
    \hfill
    \begin{subfigure}[b]{0.3\textwidth}
        \centering
        \includegraphics[width=\textwidth]{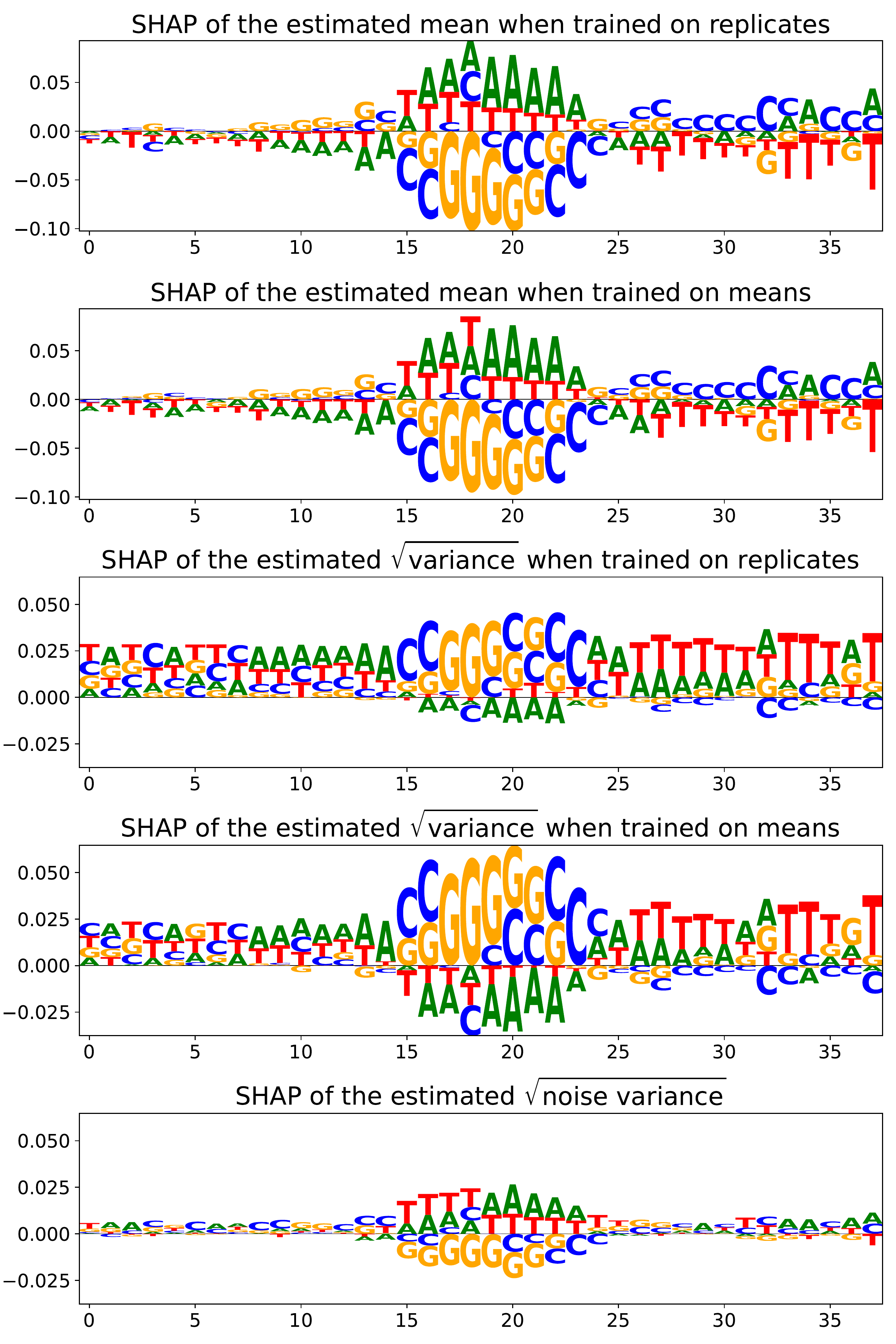}
        \caption{Heteroscedastic}
    \end{subfigure}
    \hfill
    \begin{subfigure}[b]{0.3\textwidth}
        \centering
        \includegraphics[width=\textwidth]{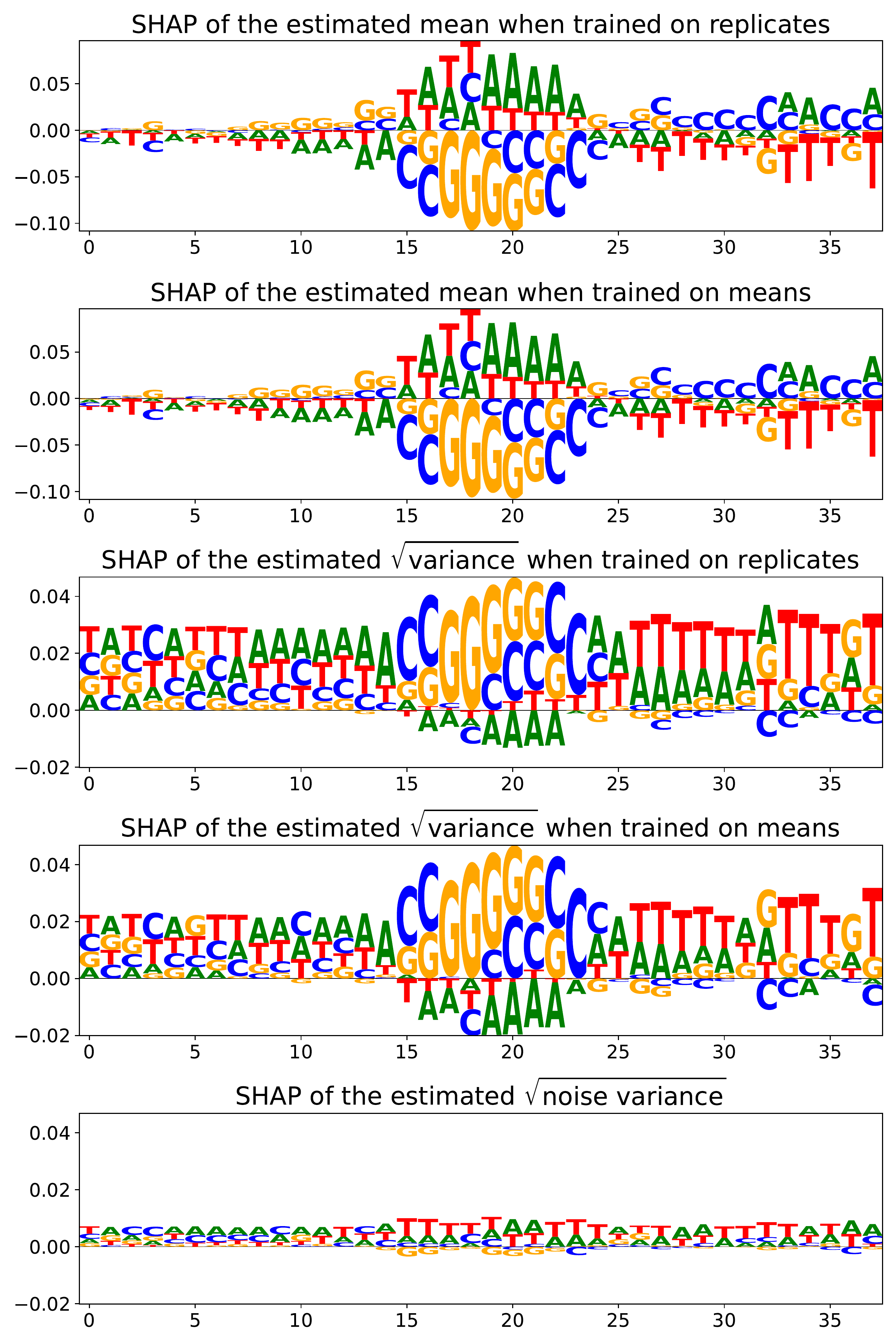}
        \caption{Beta NLL (0.5)}
    \end{subfigure}
    \\
     \begin{subfigure}[b]{0.3\textwidth}
        \centering
        \includegraphics[width=\textwidth]{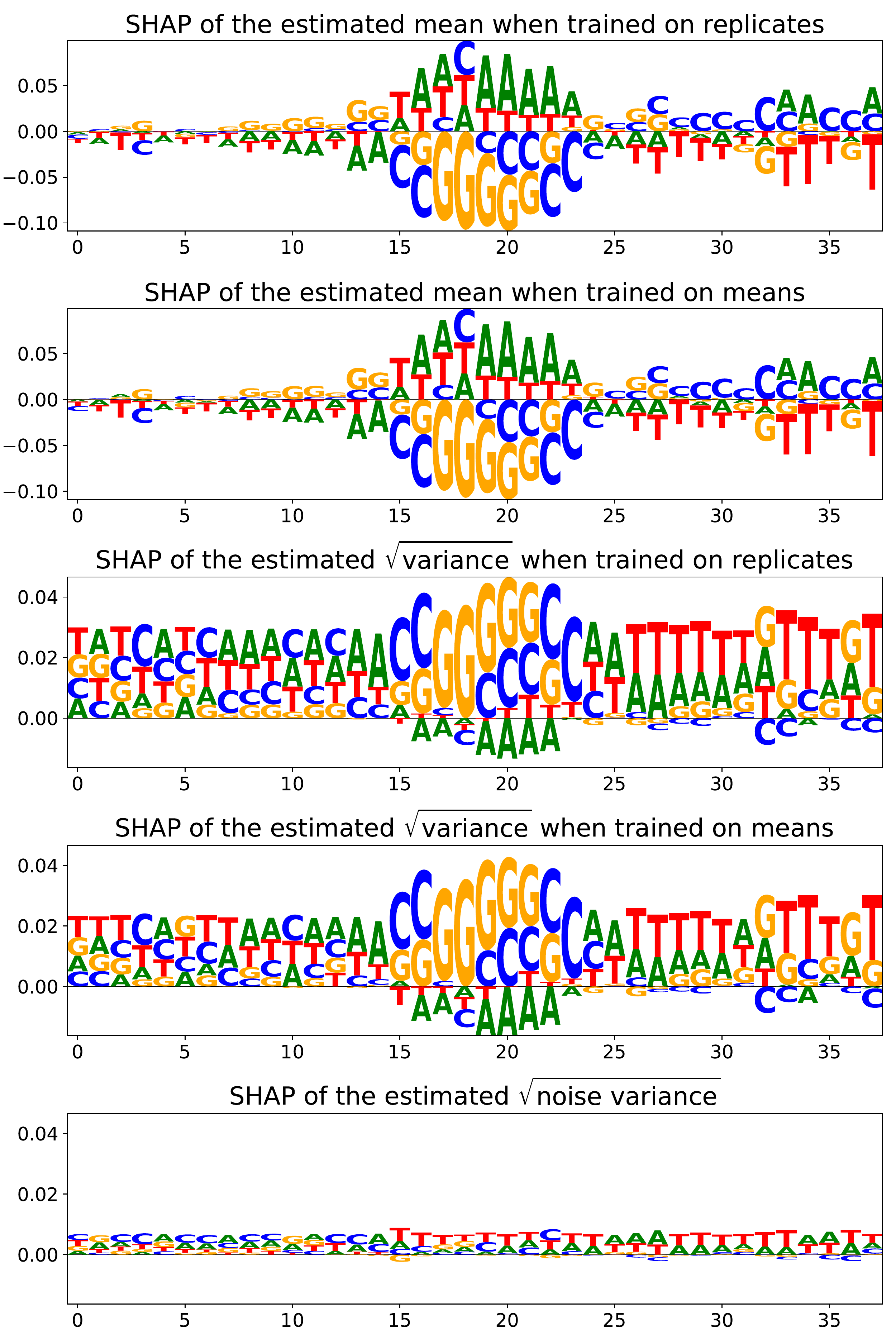}
        \caption{Beta NLL (1.0)}
    \end{subfigure}
    \vspace{20pt}
    \begin{subfigure}[b]{0.3\textwidth}
        \centering
        \includegraphics[width=\textwidth]{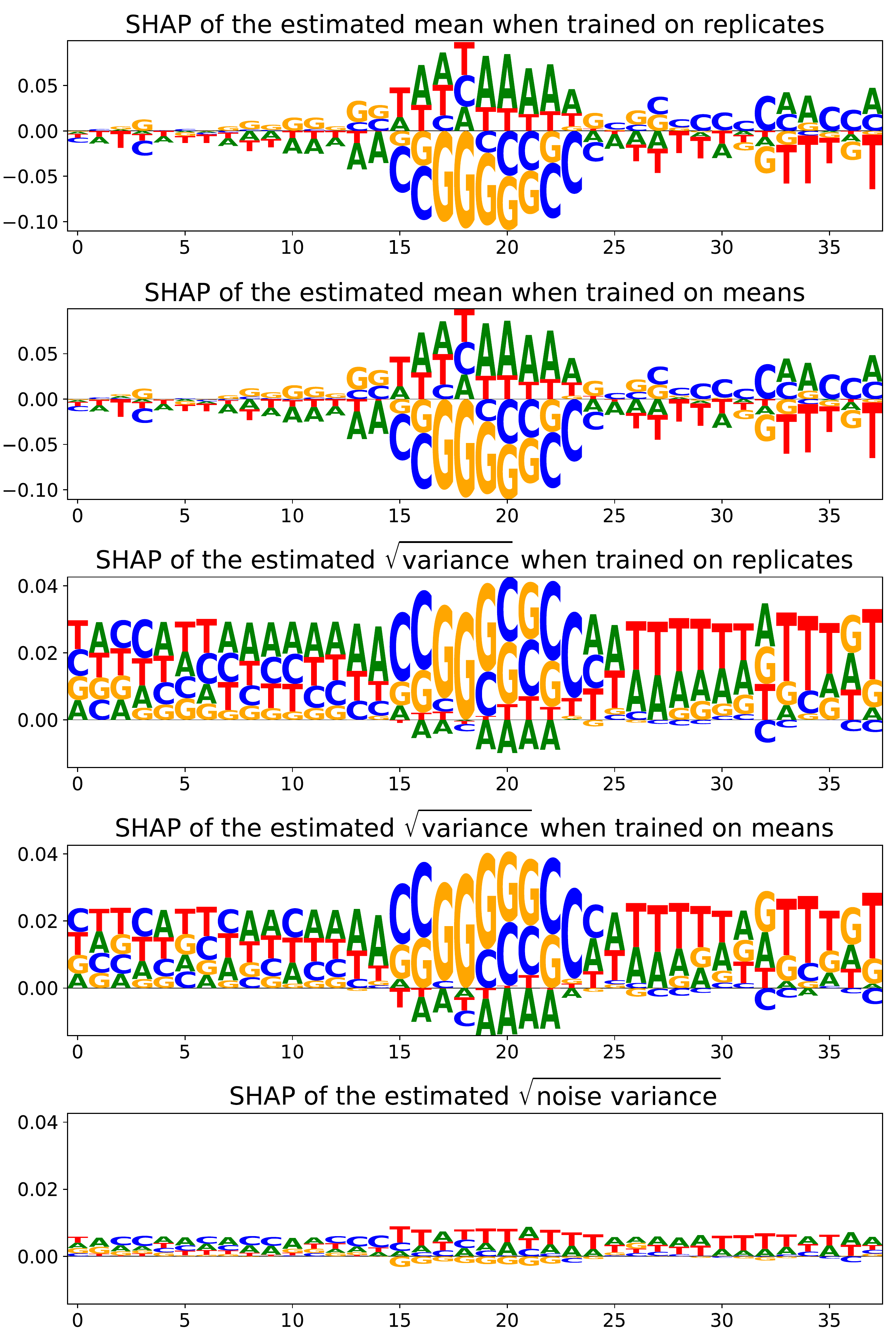}
        \caption{Faithful Heteroscedastic}
    \end{subfigure}
    \caption{SHAP Values for Survival Screen (HEK293) Dataset}
    \label{fig:crispr-off-target}
\end{figure}

}\fi

\end{document}